\documentclass{article}

\usepackage{fullpage}

\usepackage{amsfonts}  
\usepackage{amsmath}
\usepackage{amssymb}
\usepackage{amsthm}
\usepackage{appendix}
\usepackage{bbold}
\usepackage{booktabs}  
\usepackage{color}
\usepackage{courier}
\usepackage{enumitem}  
\usepackage{environ}  
\usepackage{etoolbox}  
\usepackage[T1]{fontenc}  
\usepackage{graphicx}
\usepackage{helvet}
\usepackage[utf8]{inputenc}  
\usepackage{mathtools}
\usepackage{mdwlist}  
\usepackage{microtype}  
\usepackage{multirow}
\usepackage[numbers, compress]{natbib}
\usepackage{nicefrac}  
\usepackage{tabularx}
\usepackage{times}
\usepackage{url}  
\usepackage{verbatim}
\usepackage{wrapfig}
\usepackage{xspace}

\usepackage{algorithm}
\usepackage{algorithmic}

\usepackage{hyperref}  
\usepackage{changepage}

\hypersetup{colorlinks = false, pdfborder = {0 0 0}}

\def\BibTeX{{\rm B\kern-.05em{\sc i\kern-.025em b}\kern-.08emT\kern-.1667em\lower.7ex\hbox{E}\kern-.125emX}}

\makeatletter
\def\thm@space@setup{%
  \thm@preskip=0em
  \thm@postskip=\thm@preskip
}
\makeatother

\newif\ifdraft
\draftfalse

\newif\ifshowappfootnote
\showappfootnotefalse
\newcommand{\appfootnote}{%
  \unskip
  \ifshowappfootnote
    \footnote{Appendices can be found in the supplementary material.}
    \global\showappfootnotefalse
  \fi
  \unskip
}

\renewcommand{\eqref}[1]{Equation~\ref{eq:#1}\xspace}
\newcommand{\secref}[1]{Section~\ref{sec:#1}\xspace}
\newcommand{\appref}[1]{Appendix~\ref{app:#1}\appfootnote\xspace}
\newcommand{\figref}[1]{Figure~\ref{fig:#1}\xspace}
\newcommand{\tabref}[1]{Table~\ref{tab:#1}\xspace}

\newcommand{\thmref}[1]{Theorem~\ref{thm:#1}\xspace}
\newcommand{\lemref}[1]{Lemma~\ref{lem:#1}\xspace}
\newcommand{\corref}[1]{Corollary~\ref{cor:#1}\xspace}

\newcommand{\eqrefs}[3][and]{Equations~\ref{eq:#2} #1~\ref{eq:#3}\xspace}
\newcommand{\secrefs}[3][and]{Sections~\ref{sec:#2} #1~\ref{sec:#3}\xspace}

\newtheoremstyle{mytheoremstyle}
  {\parsep} 
  {} 
  {\itshape} 
  {} 
  {\bfseries} 
  {.} 
  {0.5em} 
  {} 
\theoremstyle{mytheoremstyle}

\makeatletter
\newtheorem*{rep@theorem}{\rep@title}
\newcommand{\newreptheorem}[2]{\newenvironment{rep#1}[1]{\def\rep@title{#2 \ref{##1}}\begin{rep@theorem}}{\end{rep@theorem}}}
\makeatother

\newtheorem{theorem}{Theorem}
\newtheorem{lemma}{Lemma}
\newtheorem{corollary}{Corollary}

\newreptheorem{theorem}{Theorem}
\newreptheorem{lemma}{Lemma}
\newreptheorem{corollary}{Corollary}

\newif\ifreptheorem
\reptheoremfalse
\newif\ifshowproofs
\showproofsfalse
\newcommand{\switchreptheorem}[2]{
  \ifreptheorem
    #1\xspace
  \else
    #2\xspace
  \fi
}

\newcommand{\replabel}[1]{
  \ifreptheorem
    \tag{\ref{#1}}
  \else
    \label{#1}
  \fi
}
\NewEnviron{thm}[1]{
  \makeatletter
  \ifcsname defined:thm:#1\endcsname
    \reptheoremtrue
    \begin{reptheorem}{thm:#1}\BODY\end{reptheorem}
    \reptheoremfalse
  \else
    \begin{theorem}\label{thm:#1}\BODY\end{theorem}
    \csgdef{defined:thm:#1}{}
  \fi
  \makeatother
}
\NewEnviron{lem}[1]{
  \makeatletter
  \ifcsname defined:lem:#1\endcsname
    \reptheoremtrue
    \begin{replemma}{lem:#1}\BODY\end{replemma}
    \reptheoremfalse
  \else
    \begin{lemma}\label{lem:#1}\BODY\end{lemma}
    \csgdef{defined:lem:#1}{}
  \fi
  \makeatother
}
\NewEnviron{cor}[1]{
  \makeatletter
  \ifcsname defined:cor:#1\endcsname
    \reptheoremtrue
    \begin{repcorollary}{cor:#1}\BODY\end{repcorollary}
    \reptheoremfalse
  \else
    \begin{corollary}\label{cor:#1}\BODY\end{corollary}
    \csgdef{defined:cor:#1}{}
  \fi
  \makeatother
}
\NewEnviron{prf}[1]{
  \ifshowproofs
    \begin{proof}\BODY\end{proof}\label{app:prf:#1}
  \else
    \begin{proof}In \appref{prf:#1}.\end{proof}
  \fi
}

\newenvironment{titled-paragraph}[1]{\textbf{#1:}}{}

\newcommand{\tensorflow}{TensorFlow\xspace}

\newcommand{\iid}{\emph{i.i.d.}\xspace}
\newcommand{\wrt}{w.r.t.\xspace}
\newcommand{\ie}{i.e.\xspace}
\newcommand{\eg}{e.g.\xspace}
\newcommand{\egcite}[1]{\citep[\eg][]{#1}}

\newcommand{\probability}{\ensuremath{\mathrm{Pr}}}
\newcommand{\expectation}{\ensuremath{\mathbb{E}}}

\DeclareMathOperator*{\argmax}{argmax}

\renewcommand{\hat}{\widehat}

\renewcommand{\>}{{\rightarrow}}

\newcommand{\I}{{\mathbf I}}
\newcommand{\1}{{\mathbf 1}}
\newcommand{\0}{{\mathbf 0}}

\newcommand{\cH}{{\mathcal H}}

\newcommand{\cX}{{\mathcal X}}

\newcommand{\p}{{\mathbf p}}

\renewcommand{\v}{{\mathbf v}}

\newcommand{\x}{{\mathbf x}}
\begin{document}

\label{document:begin}

\title{Distilling Double Descent}





\author{Andrew Cotter \quad\quad Aditya Krishna Menon  \quad\quad Harikrishna Narasimhan\\Ankit Singh Rawat \quad\quad Sashank J. Reddi \quad\quad Yichen Zhou\footnote{Correspondence to: \url{yichenzhou@google.com}}\\[10pt]Google Research, USA
}



\maketitle



\begin{abstract}
Distillation is the technique of training a ``student'' model based on examples that are labeled by a separate ``teacher'' model, which itself is trained on a labeled dataset. The most common explanations for why distillation ``works'' are predicated on the assumption that student is provided with \emph{soft} labels, \eg probabilities or confidences, from the teacher model. In this work, we show, that, even when the teacher model is highly overparameterized, and provides \emph{hard} labels, using a very large held-out unlabeled dataset to train the student model can result in a model that outperforms more ``traditional'' approaches.

Our explanation for this phenomenon is based on recent work on ``double descent''. It has been observed that, once a model's complexity roughly exceeds the amount required to memorize the training data, increasing the complexity \emph{further} can, counterintuitively, result in \emph{better} generalization. Researchers have identified several settings in which it takes place, while others have made various attempts to explain it (thus far, with only partial success). In contrast, we avoid these questions, and instead seek to \emph{exploit} this phenomenon by demonstrating that a highly-overparameterized teacher can avoid overfitting via double descent, while a student trained on a larger independent dataset labeled by this teacher will avoid overfitting due to the size of its training set.
\end{abstract}

\section{Introduction}\label{sec:introduction}

The classical view of the trade-off between model complexity and performance tells us that, while more complex models are better able to represent the training data, they will also tend to find and depend upon spurious patterns that are absent in the underlying problem, but are present in the particular training sample by happenstance. Consequently, a highly-complex model will perform poorly during evaluation on held-out examples, despite performing well on the training set, since it will have learned to rely on patterns that \emph{aren't really there}. In other words, complex models will \emph{overfit}, and more complex models will overfit \emph{more}. Mysteriously, however, it has long been observed that overparameterized neural networks tend to perform better than they ``should'': some are trained to achieve near-zero training loss (implying that their capacity exceeds the amount of information in the training data), yet they still perform extremely well on held-out examples~\citep{Advani:2020,Neyshabur:2019,Geiger:2019}.

In recent years, this phenomenon has been mapped-out more completely in the ``double descent'' literature~\citep{Belkin:2019,Poggio:2019,Nakkiran:2020}. If one creates a plot with test error on the vertical axis, and some measure of model complexity on the horizontal axis (\eg the number of hidden neurons), then one observes that the classical intuition is correct, but only \emph{at first}: starting from the simplest model, the error initially improves as the complexity increases (because the ``bias'' is decreasing), until it starts to overfit, and the performance of increasingly-complex models degrades (because the ``variance'' is increasing). We call this the ``classical regime'', and---at least in some settings---it eventually breaks down: when the model is roughly complex enough to memorize the training data, the test error begins to \emph{shrink} as the model becomes still more complex. In fact, the most complex model can often be the best-performing, achieving an even lower test error than the best model in the classical regime (which we will call the ``bias-variance trade-off model''). \figref{binary-mnist}, which will be discussed in more detail later, includes two curves illustrating this phenomenon.

A number of attempts have been made to explain why double descent is observed~\egcite{Neyshabur:2019,Geiger:2019,Poggio:2019,Poggio:2020,Chang:2020}, but while such papers make compelling arguments, none are definitive. This, however, is a question that we sidestep entirely: we make no attempt to justify or explain the anomalously good performance of highly-complex models, nor the double descent phenomenon itself. Instead, we accept these as given, and then seek to \emph{exploit} them. Our ultimate goal is to learn a relatively simple model that does \emph{not} suffer from overfitting, by making use of distillation as follows:

\begin{enumerate*}
  \item Train a highly complex ``teacher'' model (\eg a neural network that is much more overparameterized than is necessary to memorize the training data) and \emph{assume} that its performance on held-out examples will significantly exceed that of the bias-variance trade-off model. This assumption (which, again, we make no attempt to justify, aside from noting that it has been empirically observed in a number of different settings~\egcite{Belkin:2019,Nakkiran:2020}) is crucial to the success of our proposal.
  \item Use this teacher model to label a large unlabeled dataset. Because this dataset is labeled with the output of the teacher model, instead of ground-truth labels, one would expect it to be lower-quality than the original training dataset. However, it will be \emph{much larger}.
  \item Use this new dataset to train a simpler ``student'' model. The fact that this model is trained on lower-quality labels will negatively impact its performance, but the \emph{size} of the dataset will reduce overfitting w.r.t. the teacher, or even eliminate it entirely (if we have access to unlimited unlabeled data).
\end{enumerate*}

We do not expect this student to outperform the teacher, but the teacher, being by assumption highly overparameterized (and therefore large, unwieldy, and expensive to evaluate), is likely to be impractical to deploy in a production environment. Indeed, reducing the \emph{cost} of the final model by training a simple student from a complex teacher is, perhaps, the canonical application of distillation~\egcite{Bucilua:2006,Hinton:2015}.

%
%
Our main contributions are: (i) the observation that---thanks to double descent---the benefit of reducing overfitting can outweigh the cost of using inferior labels in the above procedure, resulting in a simple student that outperforms an equivalently-simple model trained on the original training set; (ii) noting that, unlike most existing distillation work (\secref{related}), the success of this procedure does \emph{not} depend upon the teacher providing soft labels to the student, and indeed it works well even with hard labels; (iii) an initial theoretical justification for the use of hard labels in distillation (\secref{theory}); and (iv) an experimental evaluation of our proposal on computer vision tasks (\secref{experiments}).

\subsection{Illustration}\label{sec:introduction:illustration}

\begin{figure}[t]

\centering

\includegraphics[width=0.6\columnwidth]{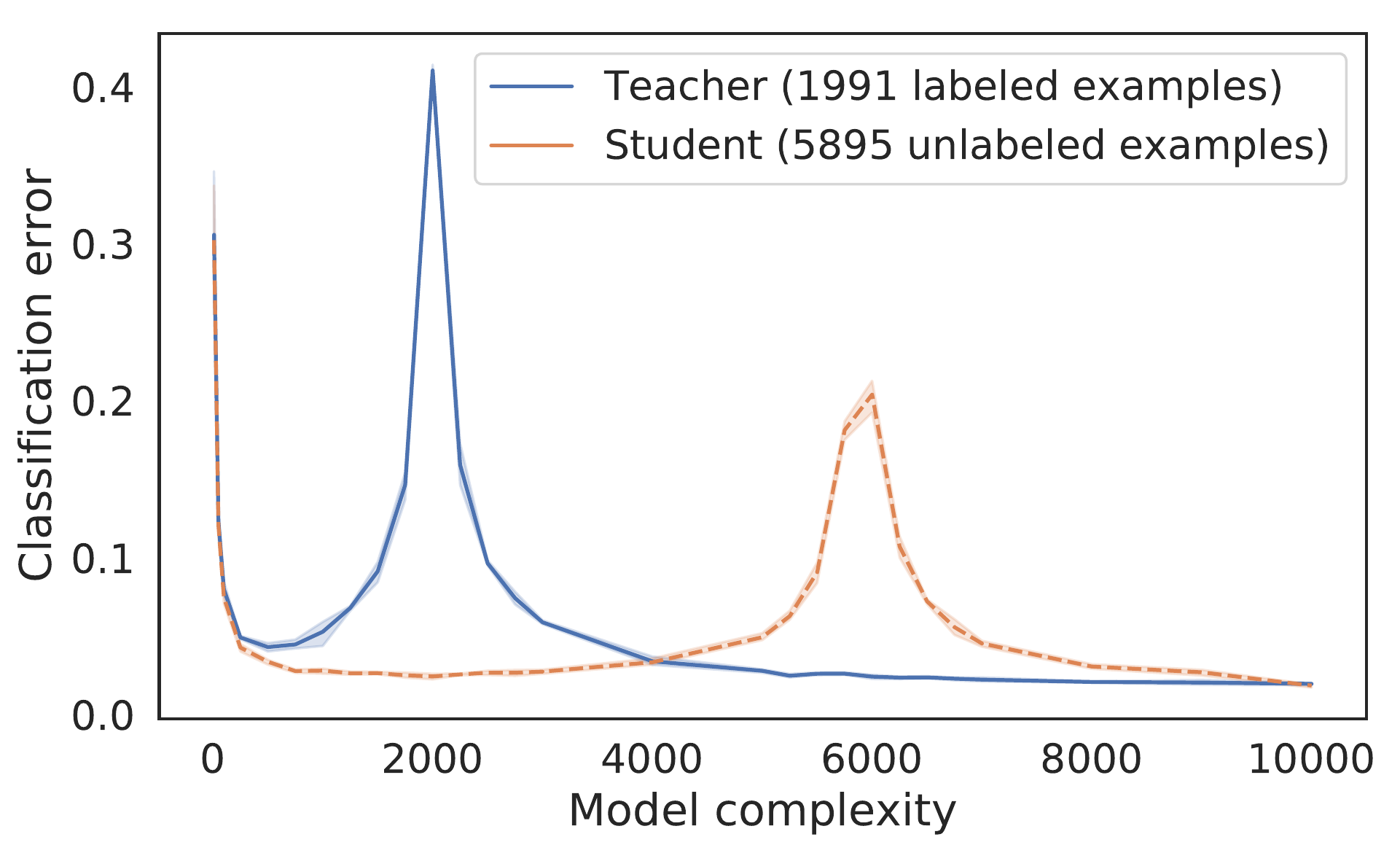}

\caption{
  Double-descent curves, plotting test error against model complexity, for the toy binary MNIST experiment of \secref{introduction:illustration}. The student is trained on an independent dataset, with soft labels being provided by the most complex teacher ($m=10000$). This teacher, despite being highly overparameterized, still generalizes well, and is therefore capable of providing high-quality labels to the (larger) student dataset. Since the student is trained on more data, it overfits less in the ``classical'' regime, and therefore outperforms equivalently-complex models trained on the original training set (\ie the ``teacher'' curve) for low complexities.
  %
}

\label{fig:binary-mnist}

\end{figure}
In \secref{experiments}, we'll present experiments on realistic image classification tasks and model architectures. Here, to clarify the above discussion, we will illustrate our proposal on a simpler ``toy'' experiment based on \citet{Belkin:2019}.

We binarize the MNIST digit classification problem by seeking to to distinguish ``4''s from ``9''s, and minimize $L_2$-regularized squared loss over random ReLU feature models of the form $f_m(x) = \sum_{k=1}^m\alpha_k\max\{\langle \x, \v_k\rangle, 0\}$, where $\v_1,\ldots, \v_m$ are random features drawn from a Gaussian distribution $\mathcal{N}(\0, \I / \sigma^2)$. We set the regularization strength to $10^{-4}$ and $\sigma$ to 5.
Here, as in all of our experiments, both teacher and student models are identically-structured, with only the complexity $m$ varying.

The teacher is trained on $1991$ labeled examples, and the student on $5895$ examples labeled by a  teacher model with complexity $m=10000$. 
We create double-descent curves for the teacher and the student: in \figref{binary-mnist}, the horizontal axis is identical for both curves (the number of random ReLU features $m$),
and we can see that, since the student models are trained on a larger dataset, the overfitting ``hump'' in the double-descent curve occurs later, with the consequence that simpler student models---\ie up to about $m=4000$ or so---outperform identically-structured (``teacher'') models that were trained on the original labeled dataset.
\section{Distillation}\label{sec:related}

The overall goal of our work---using a complex teacher to train a simpler student, because the teacher itself is too expensive to be deployed in a production system---is firmly entrenched in the mainstream distillation literature. However, the ``complex teacher'' is typically taken to be an ensemble~\egcite{Bucilua:2006,Hinton:2015,Radosavovic:2018}, instead of an overparameterized neural network.

This difference has a somewhat non-obvious consequence: because our teacher memorizes the training data (and \emph{then some}), it could make nearly-hard predictions even on held-out examples, since, during training, increasing its overall confidence is only to its benefit. Indeed, while our experiments (\secref{experiments}) show that using soft labels from an overparameterized teacher results in a better student than hard labels, we make no special effort to train a teacher that gives high-quality soft predictions (if anything, we do the opposite), and in fact, consistently with our (double-descent-based) explanation for the teacher's performance, thresholded hard labels also work \emph{very well}.

This conflicts with most of the existing distillation literature, which postulates that distillation ``works'' largely because the teacher provides good soft labels~\egcite{Lopez-Paz:2016,Furlanello:2018,Tang:2020,Menon:2020,Zhou:2021}, containing what \citet{Hinton:2015} call ``dark knowledge'', which can be coarsely understood as extra information that the teacher reveals to the student in its particular confidence scores on each example. We do not seek to challenge this conventional wisdom, and believe it to be largely correct in the settings that these papers consider, particularly given that they generally train the student on the same dataset as was used to train the teacher (albeit with teacher-derived labels, instead of the originals). Indeed, in the self-distillation regime, in which the student and teacher models additionally have exactly the same structure~\egcite{Furlanello:2018,Mobahi:2020,Zhang:2020}, it's hard to imagine any other mechanism by which the teacher could be imparting useful extra information to the student.

Our work, however, relies on a different mechanism entirely: double descent, in which an overparameterized teacher, which may be thresholded to make hard predictions, is capable of labeling fresh examples almost as well as the original labeling process. In other words, rather than exploiting the superior representational power of soft labels, we exploit the superior raw performance of overparameterized models. By using this teacher to label a large unlabeled dataset, we can train a student that overfits little \wrt the teacher. The use of an unlabeled dataset is \emph{slightly} unusual, but is not novel, having been used effectively in \eg \citet{Bucilua:2006,Hinton:2015,Radosavovic:2018}.

We should note, however, that there is not necessarily any inherent conflict between the traditional ``dark knowledge'' understanding of distillation, and our new ``double descent'' approach: one might imagine that combining highly-overparameterized neural networks with \eg bagging~\citep{Breiman:1996}, or \citet{Radosavovic:2018}'s ensemble-creation procedure, could yield an ensemble---making soft predictions---that enjoys the benefits of both.

\nocite{Advani:2020}
\nocite{Neyshabur:2019}
\nocite{Geiger:2019}
\nocite{Belkin:2019}
\nocite{Poggio:2019}
\nocite{Nakkiran:2020}
\nocite{Poggio:2020}
\nocite{Muthukumar:2020}
\nocite{Chang:2020}

\nocite{Bucilua:2006}
\nocite{Hinton:2015}
\nocite{Radosavovic:2018}
\nocite{Furlanello:2018}
\nocite{Dong:2019}
\nocite{Mobahi:2020}
\nocite{Menon:2020}
\section{Hard Teacher Labels Suffice}\label{sec:theory}

\begin{figure*}[t]

\centering

\begin{tabular}{cc}
  \includegraphics[width=0.45\textwidth]{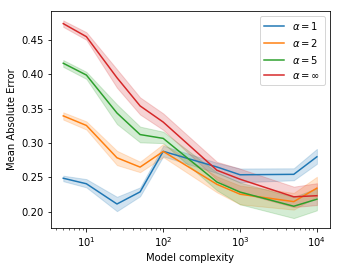} &
  \includegraphics[width=0.45\textwidth]{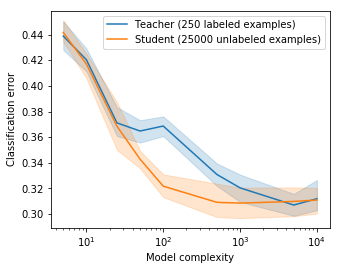}
\end{tabular}
\caption{
  Random ReLU teacher models $p^t$ trained on simulated Gaussian data, for varying numbers of ReLU features. The figure on left plots the mean absolute errors $\expectation_x\left[|p^t(x)-\phi(p^*(x))|\right]$ for different temperature scalings $\phi(z) = \frac{z^\alpha}{ z^\alpha + (1-z)^{\alpha}}$, with  $\alpha=1, 2, 5, \infty$. For $\alpha =1$, $\phi(z) = z$ and for $\alpha=\infty$, $\phi(z)=\1(z>0.5)$. The figure on right plots the classification error for the teacher models, and for a student distilled with the 10000-feature teacher.
  %
}

\label{fig:binary-simulated}

\end{figure*}
In this section, we provide one possible explanation for why distilling with a teacher making hard predictions can be as effective as using a more traditional teacher. 
A teacher making hard predictions naturally does not provide the student with the confidence scores or probabilities expected by \eg \citet{Menon:2020}, but we will show that such a teacher can still perform well when used to label a large pool of unlabeled examples, with the student being trained on this larger dataset. The condition that we need the teacher to satisfy is that its predictions approximate a ``margin-preserving'' transformation of the underlying class probabilities (with hard-thresholding being one such transformation).

To this end, we formally state our problem setup.
Let $D$  denote the underlying data distribution over instances  and labels $\cX \times [m]$, and let $D_{\cX}$ denote the marginal distribution over instances. We'll be particularly interested in the conditional-class probabilities for this distribution, which we denote by $p^*_y(x) = \probability(Y=y|X=x)$, and seek to use teacher models that approximate some transformation of these class probabilities.

We consider student classification models $h: \cX \> [m]$ 
that map instances $x$ to one of $m$ labels, and measure their performance in terms of their  expected 0-1 risk:
\begin{eqnarray}
R(h) &=&
\expectation_{(x,y)\sim D}\left[ \1(h(x) \ne y)\right]
\nonumber
\\
&=& \expectation_{x\sim D_\cX}\left[ \sum_{i=1}^m p^*_i(x) \1(h(x) \ne i)\right].
\label{eq:pop-risk}
\end{eqnarray}
Specifically, we will evaluate the student model against the classifier that achieves the minimum 0-1 risk, i.e.\ against the \emph{Bayes-optimal} classifier, which takes the form:
$$h^*(x) = \argmax_{y\in[m]} p^*_y(x),$$
where, we assume, for convenience,  that ties are broken in favor of the larger class. 

Our  approach is to learn the student classifier from a simple hypothesis class $\cH$ using a  teacher model $\p^t:\cX\>\Delta^m$, where $\Delta^m$  denotes the $(m-1)$-dimensional simplex with $m$ coordinates. We use the teacher to label a large sample of unlabeled examples $S^u =\{x_1,\ldots, x_{n^u}\}$ of size $n^u$ drawn from $D_\cX$, and seek to minimize the student's error on the teacher-labeled dataset. This procedure can be stated as an empirical risk minimization (ERM) problem for the student:
\begin{equation}
    \hat{h} \in \min_{h\in \cH}\, \frac{1}{n^u}\sum_{x \in S^u}\sum_{i=1}^m p^t_i(x)\1(h(x) \ne i),
\label{eq:erm}
\end{equation}
where in practice, the indicator functions are approximated with differentiable surrogate losses.
We will find it useful to also define the expected \emph{teacher-distilled} risk for a student classifier:
\begin{eqnarray}
R^t(h) &=& \expectation_{x\sim D_\cX}\left[ \sum_{i=1}^m p^t_i(x) \1(h(x) \ne i)\right].
\label{eq:distilled-risk}
\end{eqnarray}


We are now ready to state our main result.
Our core premise is that while the over-parameterized teacher $\p^t$ may not provide \emph{soft} labels that
approximate the conditional-class probabilities $\p^*$ well, it still may be able to provide \emph{nearly-hard} labels that approximate a certain transformation $\phi: \Delta^m \> \Delta^m$ of the  class probabilities $\p^*$, i.e.\ $\p^t(x) \approx \phi(\p^*(x))$. Below, we bound the excess risk for the student empirical risk minimizer $\hat{h}$ of \eqref{erm}, when the transformation $\phi$ preserves the margin of separation between the most-likely label and all other labels.
\begin{thm}{excess-error-bound}
Let $\phi: \Delta^m \> \Delta^m$ be a transformation function such that the  transformed probabilities $\p^\phi(x) = \phi(\p^*(x))$ satisfy: 
\begin{equation}
  \replabel{eq:argmax}
  \argmax_{y} p^\phi_y(x) = \argmax_{y} p^*_y(x),
\end{equation}
where ties are broken in favor of the larger class, and
\begin{equation}
  \replabel{eq:margin}
  \max_i p^\phi_i(x) - p^\phi_j(x) \,\geq\, \max_i p^*_i(x) - p^*_j(x), \forall j \in [m].
\end{equation}
Fix $\delta \in (0,1)$. Then with  probability at least $1 - \delta$ over draw of $n^u$ unlabeled examples from $D_\cX$, the solution $\hat{h} \in \cH$ to the student empirical risk minimization problem in \eqref{erm} with a fixed teacher $\p^t$ satisfies:
\switchreptheorem{
  \begin{equation*}
    R(\hat{h}) - R(h^*) \le 
    \underbrace{\mathcal{O}\left(\sqrt{\frac{\log(|\cH|/\delta)}{n^u}}\right)}_{\text{Student estimation error}}
    + \underbrace{\min_{h\in \cH}R^t(h) - \min_{h: \cX \>[m]}R^t(h)}_{\text{Student approximation error}}
    + \underbrace{\expectation_{x}\left[\|\p^t(x) - \p^\phi(x)\|_1\right]}_{\text{Teacher approximation error}},
  \end{equation*}
}{
  \begin{equation*}
    R(\hat{h}) - R(h^*) \le 
    \underbrace{\mathcal{O}\left(\sqrt{\frac{\log(|\cH|/\delta)}{n^u}}\right)}_{\text{Student estimation error}}
    + \underbrace{\min_{h\in \cH}R^t(h) - \min_{h: \cX \>[m]}R^t(h)}_{\text{Student approximation error}}
    + \underbrace{\expectation_{x}\left[\|\p^t(x) - \p^\phi(x)\|_1\right]}_{\text{Teacher approximation error}},
  \end{equation*}
}
where $|\cH|$ can be replaced by a measure of capacity of the student hypothesis class $\cH$.
\end{thm}
%


Each of the  three terms captures an important aspect of the learning problem: (i) The student estimation error represents the variance in the student risk, and for a hypothesis class $\cH$ with finite capacity, decreases as we increase the size of the unlabeled sample $n^u$.  (ii) The student approximation error measures how close we can get to the optimal distilled risk $R^t$ using  models in $\cH$. 
(iii) The teacher approximation error  captures how closely the teacher $\p^t$ approximates a $\phi$-transformed version of the class probabilities, where $\phi$ can be any transformation that satisfies the conditions in \eqrefs{argmax}{margin}. One can tighten this bound by taking a min over all such transformations $\phi$, or equivalently over all margin-preserving projections of $p^*$.

Note that $\phi(z) = z$ would trivially satisfy the margin conditions in \eqrefs{argmax}{margin},  and recovers the conditional-class probabilities. More generally, any temperature-scaling transformation $\phi_i(z) = z_i^\alpha / \sum_j z_j^\alpha$, for $\alpha > 1$, would satisfy these conditions
(see \appref{temperature-scaling} for a proof), and so would a  hard-thresholding transformation  $\phi_i(z) = \1(\argmax_j z_j = i)$. 

We note that
\citet{Menon:2020} also offer a ``statistical perspective'' for distillation, with one key difference being that, unlike us, they assume the teacher and student models are trained on the same dataset.
They show that when the teacher approximates the underlying conditional-class probabilities $p^*$ well, the soft labels that it provides helps improve the student model's generalization. 
In contrast, our analysis is for a setting in which we use high quality hard labels on a large held-out unlabeled sample.

\subsection{Illustration of Teacher Approximation Error}

Our expectation is that an over-parameterized teacher could be better at approximating a hard transformation of the class probabilities $p^*$ than the exact probabilities themselves. To demonstrate this, we use a simulated 50-dimensional dataset with binary labels, where the positive and negative examples are drawn from Gaussian distributions $\mathcal{N}(0.1 \cdot \1, \I)$ and $\mathcal{N}(-0.1 \cdot \1, \I)$ respectively, with a class prior of 0.5, and for which the conditional-class probability function is $p^*(x) = \probability(y=1|x) = 1 / (1+\exp(-0.2\sum_i x_i))$. We generate 250 labeled examples from this distribution, and fit a teacher model $p^t$ by minimizing a cross-entropy loss over the random ReLU feature models described in \secref{introduction:illustration}. 

\figref{binary-simulated}(a) shows the mean absolute error $\expectation_x\left[|p^t(x)-\phi(p^*(x))|\right]$ between the teacher predictions and different temperature-scaled transformations $\phi$ of the class probabilities $p^*$, as we increase the number of ReLU features used by the model. In the classical regime, the original teacher model is better at approximating the class-probabilities $p^*$ ($\alpha=1$) than temperature-scaled versions, but, after the double descent phenomenon kicks in, the model is better at approximating transformations of  $p^*$ at larger temperatures ($\alpha\geq 5$). 
\figref{binary-simulated}(b) shows the classification errors for both the teacher model, and a student model trained on 25000 examples labeled by a $10000$-feature teacher. Clearly, the student models outperform identically-structured teacher models trained on the originally labeled examples.

\begin{table}[t]
  \centering
  \caption{Models, sample sizes and training epochs used by teacher and student models. We use fewer epochs to train the students to partially compensate for the larger sizes of their training sets.}
  \vspace{5pt}
  \begin{tabular}{r|lll}
    \hline\hline
    Dataset & CIFAR-10 & SVHN & ImageNet \\
    \hline
    Model & ResNet18 & CNN & ResNet18 \& ResNet50 \\
    \hline
    Teacher Training Examples & 16,500 & 72,526 & 256,233\\
    Student Training Examples & 33,500 & 531,862 & 1,024,934 \\
    Testing Examples & 10,000 & 26,032 & 50,000 \\
    \hline
    Teacher Training Epochs & 1,000 & 1,000 & 100 \\
    Student Training Epochs & 1,000 & 250 & 50\\\hline\hline
  \end{tabular}
  \label{tab:experiment-setup}
\end{table}
\begin{figure*}[t]

\centering

\begin{tabular}{cc}
  \includegraphics[width=0.45\textwidth]{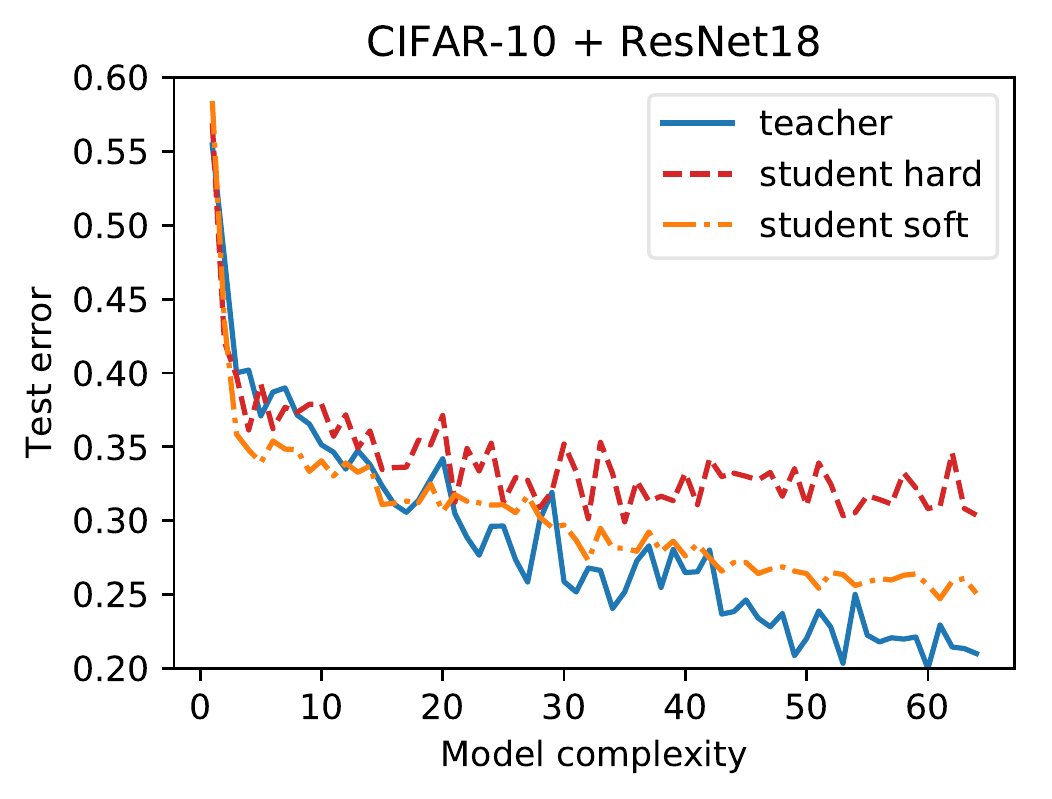} &
  \includegraphics[width=0.45\textwidth]{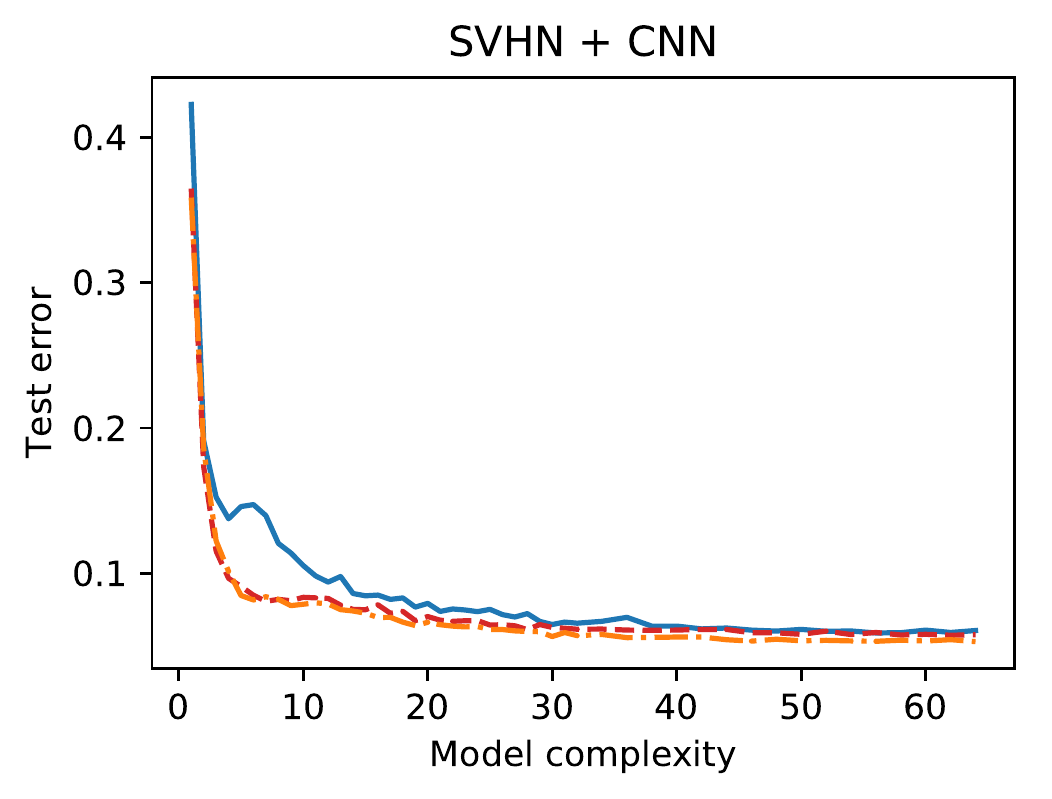} \\
  \includegraphics[width=0.45\textwidth]{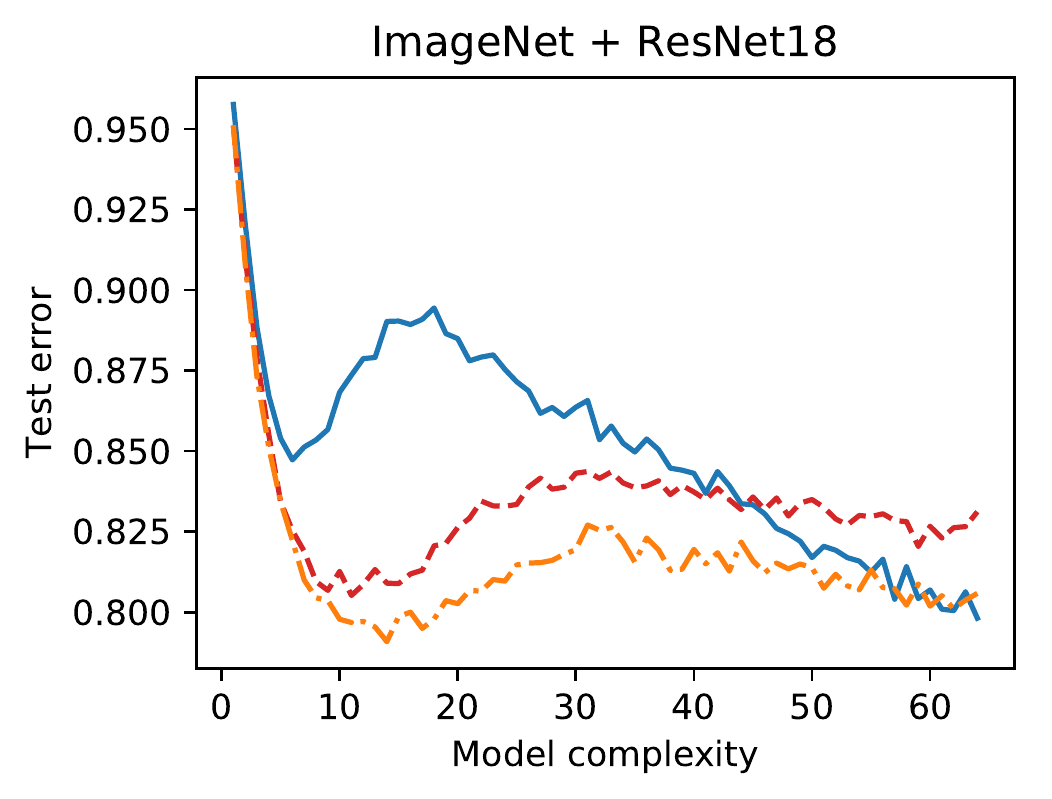} & 
  \includegraphics[width=0.45\textwidth]{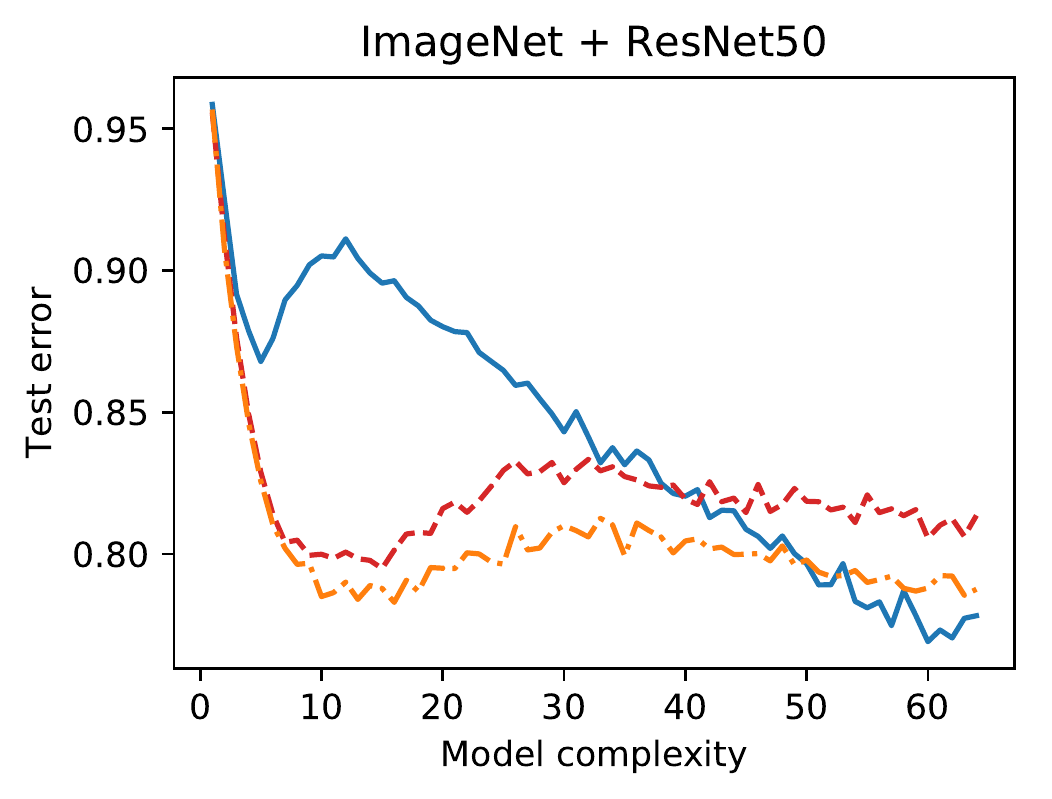}
\end{tabular}
\caption{
  Test error rates as a function of model complexity $m$ for a ``teacher'' model trained on the teacher training set, and for two student models trained on a dataset labeled by the teacher. One student was trained on the raw ``soft'' teacher labels, and the other on ``hard'' labels found by thresholding the teacher.
  Clockwise from top-left: CIFAR-10 using ResNet18 models, SVHN using 5-layer CNN models, ImageNet using ResNet50 models, and ImageNet using ResNet18 models. To avoid ambiguity, for both ImageNet experiments we chose to report top-1 classification errors instead of (the more common choice of) top-5 errors. Since ImageNet has 1000 different classes and the source images were downsized (see \citet{chrabaszcz2017downsampled}), their test error rates are relatively high. Teacher-train/student-train/test splits can be found in \tabref{experiment-setup}. 
  \vspace{-5pt}
}

\label{fig:experiments}

\end{figure*}
\begin{figure*}[t]

\centering

\begin{tabular}{cc}
    \includegraphics[width=0.45\textwidth]{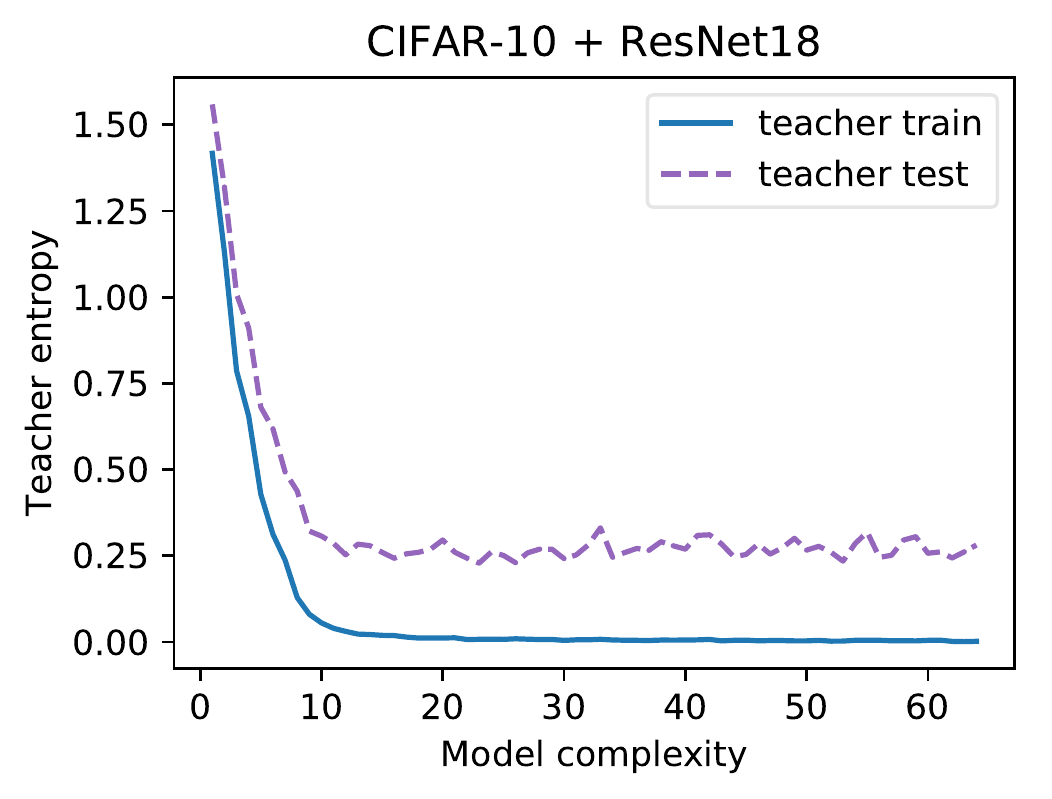} 
    & \includegraphics[width=0.45\textwidth]{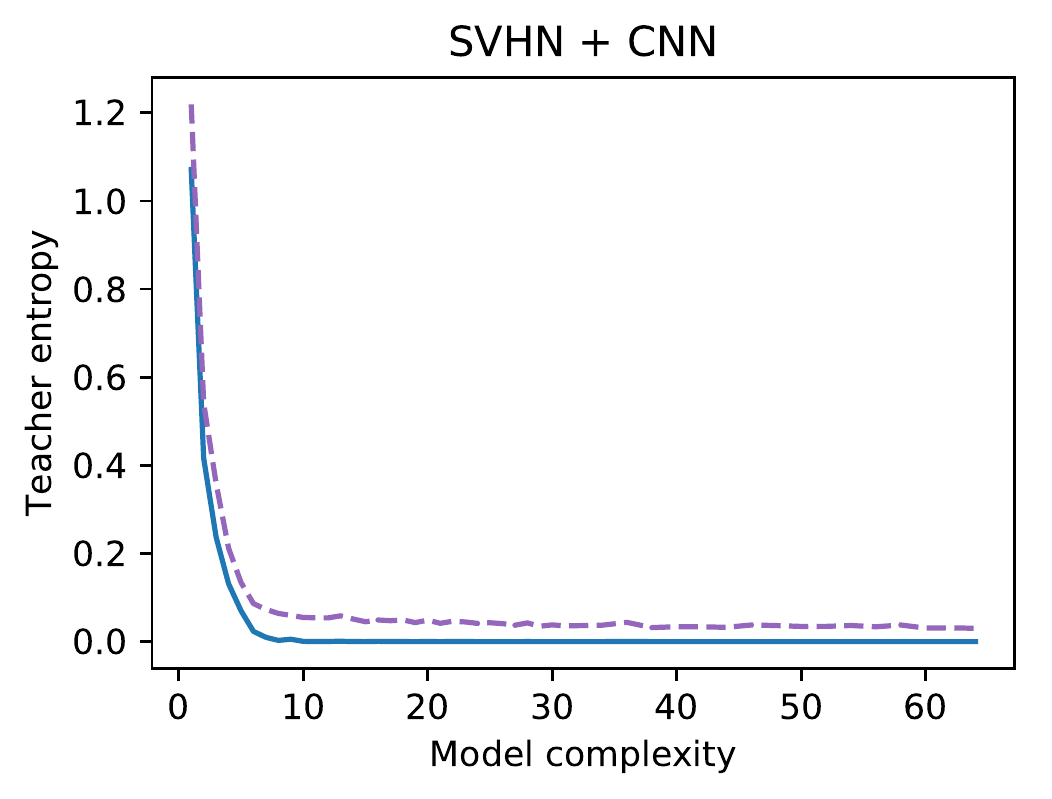} \\
    \includegraphics[width=0.45\textwidth]{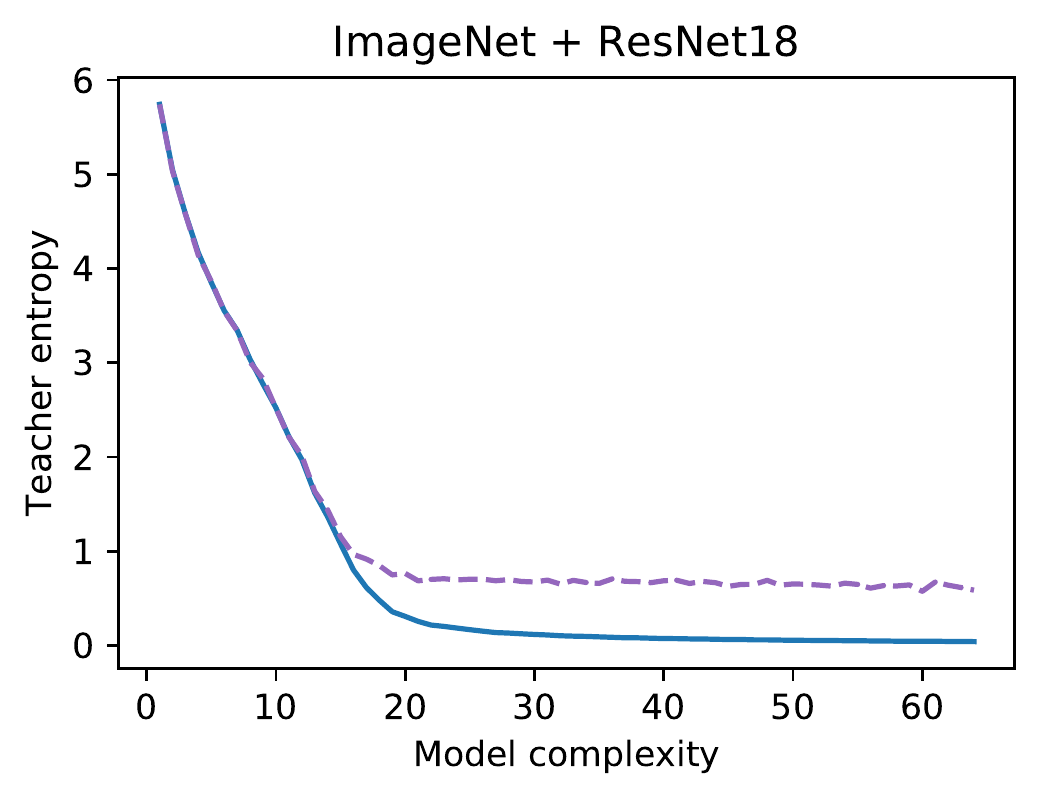} 
    & \includegraphics[width=0.45\textwidth]{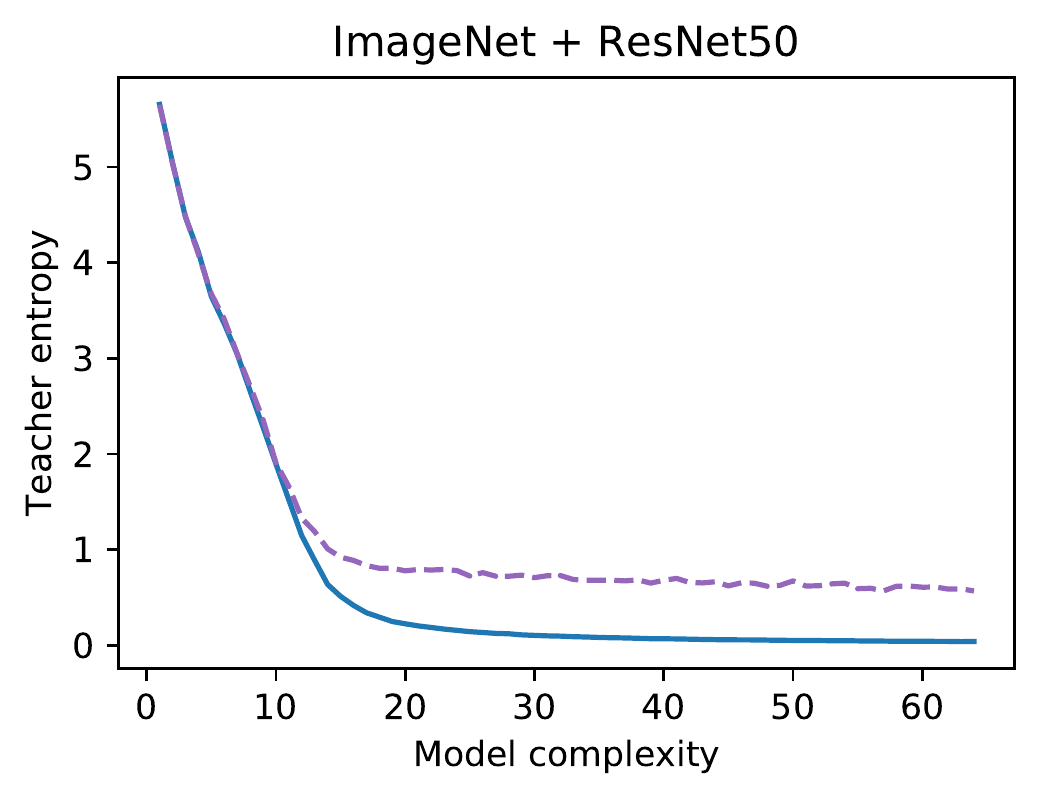}
\end{tabular}

\caption{
  Average entropy of the teacher's predictions on the training and test sets. Lower values indicate a ``harder'' classifier, and higher values a ``softer'' one. As model complexity increases, the teacher does \emph{not} seem to be converging to a hard classifier on the test set.
  Clockwise from top-left: CIFAR-10 using ResNet18 models, SVHN using 5-layer CNN models, ImageNet using ResNet50 models, and ImageNet using ResNet18 models. Teacher-train/student-train/test splits can be found in \tabref{experiment-setup}.
}

\label{fig:teacher-entropy}

\end{figure*}
\begin{figure*}[t]

\centering

\begin{tabular}{cc}
  \includegraphics[width=0.45\textwidth]{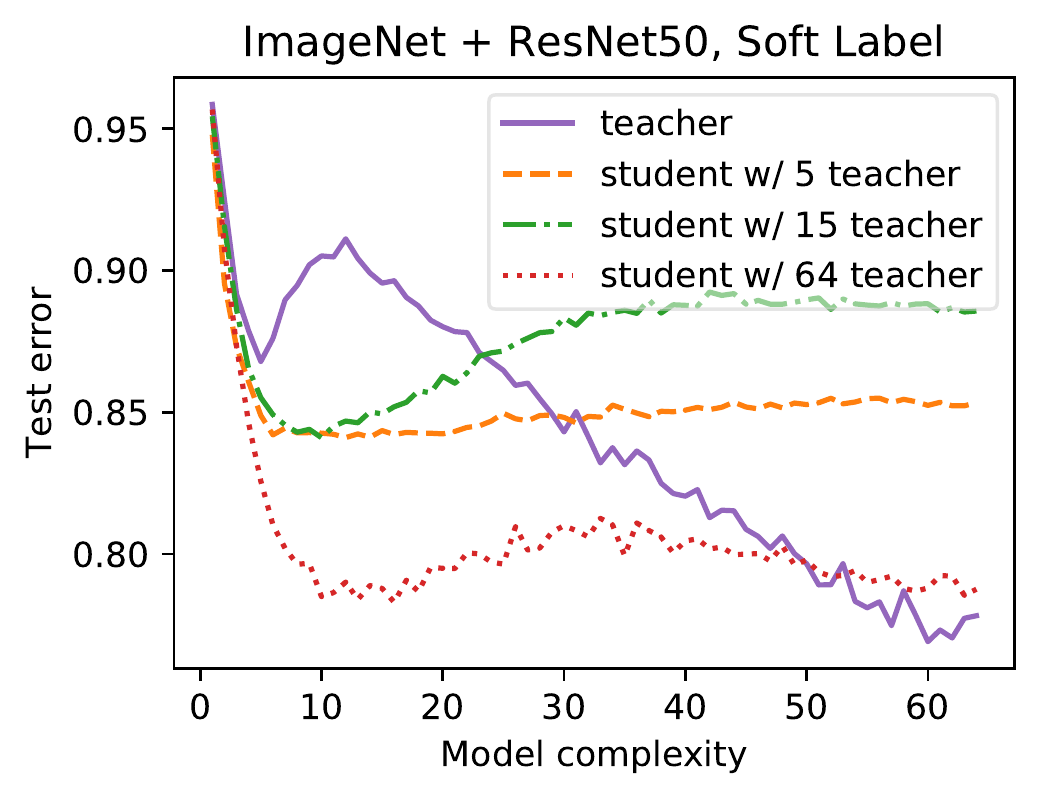} &
  \includegraphics[width=0.45\textwidth]{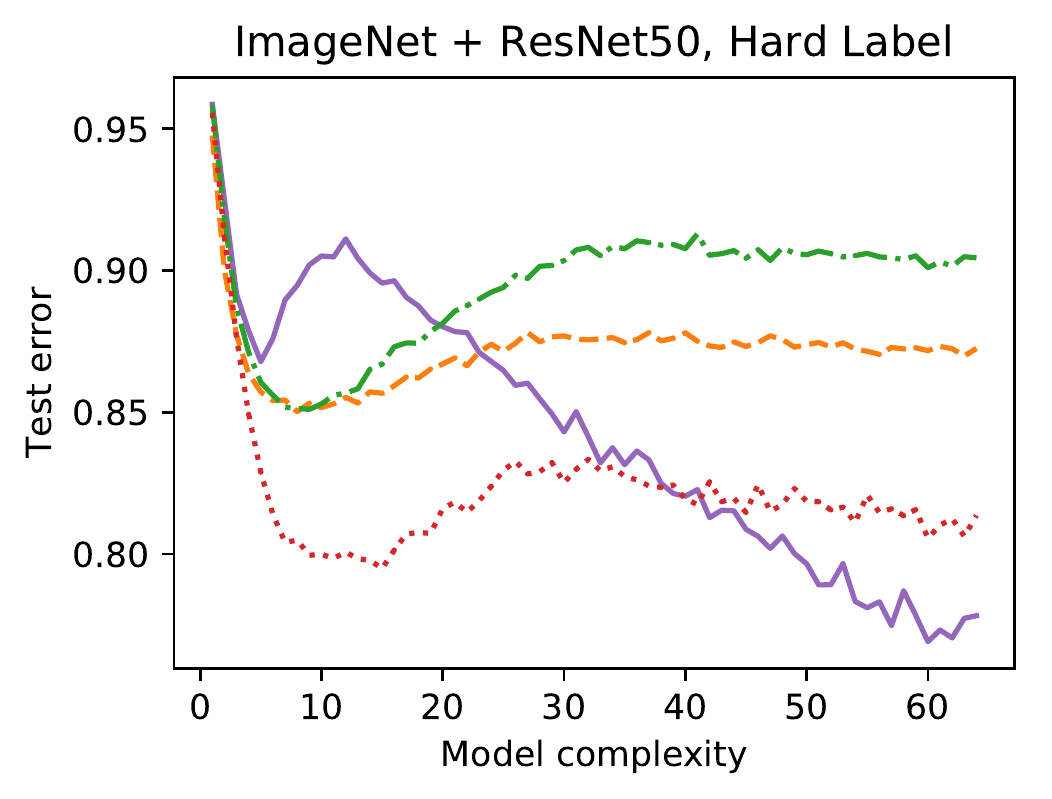}
  \vspace{-3pt}
\end{tabular}

\caption{
  Test error rates as a function of model complexity $m$ for a ``teacher'' model trained on the teacher training set, and for three student models trained on a dataset labeled by teachers of complexities $m=5$, $m=15$ or $m=64$, on ImageNet using ResNet50 models.
  Left: students are trained on the raw soft teacher labels. Right: on hard labels found by thresholding the teacher's predictions.
}

\label{fig:imagenet-three-points}

\end{figure*}
\begin{figure}[t]

\centering

\begin{tabular}{c}
\vspace{-3pt}
    \includegraphics[width=0.45\columnwidth]{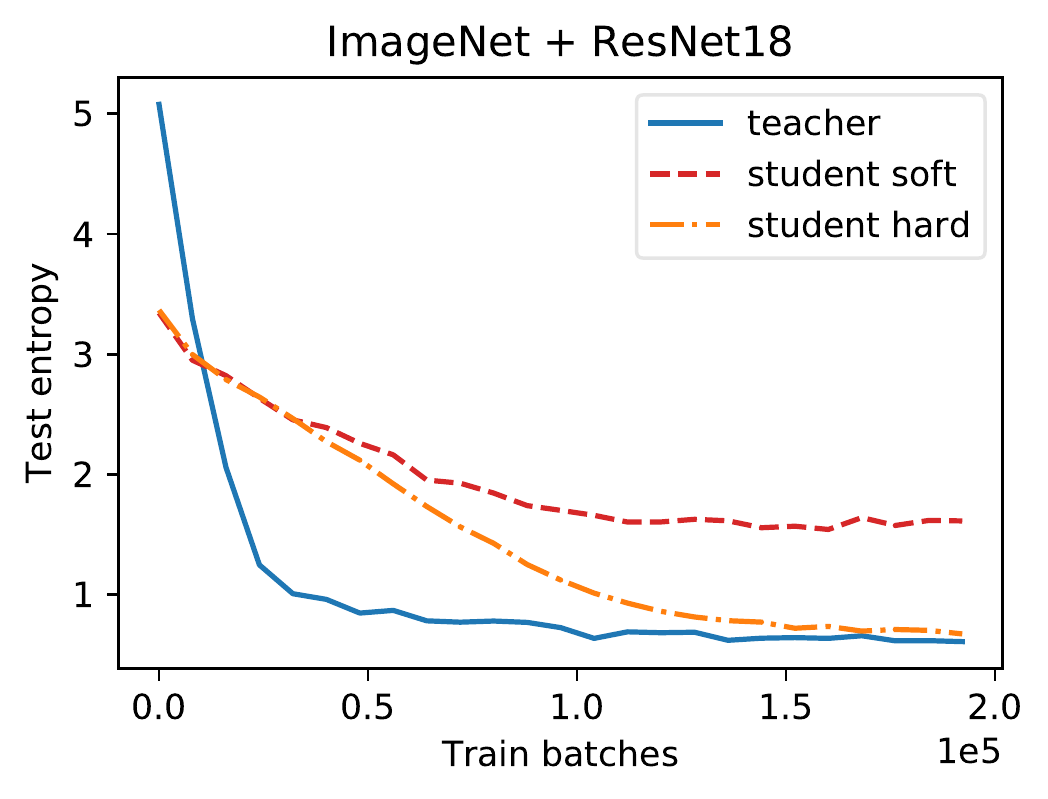} 
    \vspace{-7pt}
\end{tabular}
\caption{
  Average entropy of the most complex teacher and student classifications on the test set, as a function of \emph{training batches} instead of complexity, on ImageNet. Lower values indicate a ``harder'' classifier, and higher values a ``softer'' one. The teacher seems to be converging to a \emph{soft} classifier with entropy $\approx 0.60$.
  \vspace{-7pt}
}

\label{fig:imagenet-resnet18-test-entropy}

\end{figure}

\section{Experiments}\label{sec:experiments}

We applied our proposed double descent \& distillation procedure to three benchmark image classification tasks: CIFAR-10~\citep{Krizhevsky09learningmultiple}, ImageNet~\citep{ILSVRC15, chrabaszcz2017downsampled}, and SVHN~\citep{Netzer2011}, all of which were acquired from the \tensorflow Datasets package~\citep{TFDS}.

On CIFAR-10 we used ResNet18 networks~\citep{he2016deep}, on ImageNet we experimented with both ResNet18 and ResNet50, and on SVHN we used a 5-layer CNN. All models were trained using Adam~\citep{Kingma:2015} with a constant learning rate 0.0001. On the first two datasets, we follow \citet{Nakkiran:2020} by varying the number of filters used by the models in their intermediate layers. Specifically, for the ResNet18 and ResNet50 models we use $m$, $2m$, $4m$ and $8m$ filters for their four residual blocks (respectively). On SVHN, the four convolutional layers likewise contain $m$, $2m$, $4m$ and $8m$ filters.

In all cases, the complexity $m$---\ie the horizontal axis of our double-descent plots---varies between $m=1$ and $m=64$. In \tabref{experiment-setup}, we describe further experimental details, particularly how the dataset was divided into teacher-training, student-training and testing datasets. 

All plots are generated from single runs due to their intensive resource requirements. We provide empirical evidence that these results are stable and reproducible in \appref{experiments}, along with additional experimental details.

\subsection{Double Descent of Students}

\figref{experiments} shows error curves for the teacher, and two student models trained from the most complex teacher model (the rightmost point on the curve), for each of our four experiments. All three models, teacher and students alike, have identical structures (parameterized by $m$, on the horizontal axis), while the student models differ from each other only in whether they were provided with soft teacher labels, or thresholded hard labels.

In the SVHN and ImageNet experiments (top right, bottom left and bottom right), the teacher shows a clear double-descent pattern, especially on ImageNet. The students show weaker double-descent patterns, we believe because they were trained on many more examples than the teacher.

More significantly, except on CIFAR-10, both students (soft and hard) outperform the teacher for simpler models, indicating that, if one desires a simple model, then our proposal of training an overparameterized teacher, and then using its labels on a large unlabeled dataset to train a simple student, can be more effective than training a simple model directly on the original training set (\ie the teacher curve itself).
%

On CIFAR-10, we don't see a clear double-descent pattern in any curves, but we \emph{do} see that the test error decreases even past the point at which the data are memorized. In other words, the portion of the plot that's ``missing'' is the \emph{classical} regime, not the overparameterized regime. Since our approach depends upon the student overfitting less (and therefore overperforming) in the classical regime, one might expect that this would hurt the efficacy of our approach, and indeed the ``hard'' student does not seem to significantly outperform the teacher at any complexity level.


All four experiments show that soft labels outperform hard labels, and---even on CIFAR-10---the students trained from soft labels outperform the teachers at low complexities.

\subsection{Hard vs. Soft Teacher Labels}\label{sec:experiments:hard-soft}

As we mentioned in \secref{related}, most of the existing work on distillation uses \emph{soft} labels from the teacher to train the student. In our work, however, we rely on highly overparameterized teacher models that can memorize the training data, and therefore might expect that the teacher will make nearly-hard predictions even on held-out data. 
Indeed, as shown in \figref{teacher-entropy}, the teachers do make increasingly confident predictions, but beyond a certain point they do \emph{not} become more confident on the test set. Indeed, we found that while thresholded hard labels \emph{do} work well, soft labels work even \emph{better}, even though we have no reason to expect them to be interpretable as \eg probabilities. The ``best of both worlds'' approach, therefore, would seemingly be to use soft labels from a highly overparameterized teacher.
%

The fact that the teacher even produces soft labels is somewhat surprising, since complex teachers memorize the training data (this occurs at roughly $m=10$ for CIFAR-10 and SVHN, and $m=20$ for both ImageNet experiments--see the training curves in \figref{teacher-entropy}), and one would therefore expect them to become increasingly confident as the complexity increases, even on held-out examples. 
One might contend that the reason we are not seeing this is due to early stopping, \ie that we terminated optimization before it had the opportunity to converge to a hard classifier. However, \figref{imagenet-resnet18-test-entropy} shows that, even as a function of \emph{training time}, the teacher seems to be converging to a soft classifier.

We speculate that the reason that the teacher can provide 
soft labels even when it is more-than-capable of memorizing its training set is related to the reason that overparameterized models generalize unexpectedly well in the first place: perhaps it is identifying and avoiding overconfidence in regions of true uncertainty. 
Some additional evidence for this hypothesis can be found in \figref{train-test-loss} in \appref{experiments}, in which we see that, beyond a certain point, the cross-entropy loss of the teacher on the test set starts to \emph{decrease} as the complexity increases (although it is never competitive with the best ``classical'' model), which wouldn't happen if the teacher was increasingly overconfident on misclassified testing examples. 
Beyond raising this possibility, however, we offer no explanation for this phenomenon.

\subsection{Comparison with Simpler Teachers}

It's natural to ask how our proposed approach compares to the traditional view of distillation, in which the teacher model does not merely provide soft labels, but is chosen to provide \emph{high-quality} soft labels~\egcite{Menon:2020,Zhou:2021}---often taken to be probabilities---which can be accomplished by \eg controlling the model complexity or adding sufficient regularization.

To attempt to answer this question, we compared our proposal---using the most complex teacher that we trained (\ie the rightmost point on the double-descent curve) to train the student---against simpler trainers, chosen from earlier on the double-descent curve,
that have a lower log loss (see \figref{train-test-loss} in \appref{experiments}), indicating that their predictions are likely to be higher-quality probability estimates.
\figref{imagenet-three-points} compares our proposed student (trained from the $m=64$ teacher) against two students trained from simpler teachers ($m=5$ and $m=15$) on ImageNet with ResNet50 models. The optimal bias-variance trade-off point for this teacher occurs at roughly $m=5$,
%
%
and we can see that, while all students outperform the teacher at low complexity levels, the student trained from the most complex teacher model significantly outperforms the alternatives, for both hard and soft labels. More importantly, a student trained from the most complex teacher using hard labels dominates those trained from less complex teachers, regardless of whether the labels provided to the students were hard or soft.

\section{Conclusions and Future Work}\label{sec:conclusion}

The recent excitement about double descent has highlighted the anomalously good performance of overparameterized neural networks, but it has naturally been viewed as something that is only relevant for large models. If one desires a simple model, in the classical regime, subject to the bias-variance trade-off, then why should double descent matter?

As we've shown, it \emph{matters} because the performance of a highly overparameterized model can be---at least partially---transmitted to a simpler model using distillation. If this transmission is accomplished via an unlabeled dataset that is much larger than the training set, then the simple model will overfit less, and thus significantly outperform an equivalent model that was trained on the original training set.

As with double descent, our work is also an atypical application of distillation, since most of the existing literature relies upon the teacher providing \emph{soft} labels (\secrefs{related}{theory}), whereas we have no such requirement. It's true that, when an overparameterized teacher provides soft labels to the student, the result outperforms one provided with hard labels. However, the latter student still performs very well, which indicates that, at least in our setting, the standard ``dark knowledge'' explanation of distillation is incomplete. We contend that the (still mostly-unexplained) double-descent phenomenon fills-in this gap in our understanding.

The fact that our teachers were capable of making reasonably-good soft predictions is itself somewhat surprising, and it happened by \emph{accident}, not design. This raises the question of whether one should expend extra effort to train an overparameterized teacher that also makes \emph{high-quality} soft predictions, \eg using bagging~\citep{Breiman:1996}, \citet{Radosavovic:2018}'s approach, or something else.
This is, we believe, an exciting area for future research.

\newpage
\clearpage


\bibliography{main}
\bibliographystyle{icml2021}

\label{document:middle}

\newpage
\clearpage

\appendix

\label{document:appendix}

\showproofstrue

\section{Proofs}\label{sec:proofs}

\begin{thm}{excess-error-bound}
Let $\phi: \Delta^m \> \Delta^m$ be a transformation function such that the  transformed probabilities $\p^\phi(x) = \phi(\p^*(x))$ satisfy: 
\begin{equation}
  \replabel{eq:argmax}
  \argmax_{y} p^\phi_y(x) = \argmax_{y} p^*_y(x),
\end{equation}
where ties are broken in favor of the larger class, and
\begin{equation}
  \replabel{eq:margin}
  \max_i p^\phi_i(x) - p^\phi_j(x) \,\geq\, \max_i p^*_i(x) - p^*_j(x), \forall j \in [m].
\end{equation}
Fix $\delta \in (0,1)$. Then with  probability at least $1 - \delta$ over draw of $n^u$ unlabeled examples from $D_\cX$, the solution $\hat{h} \in \cH$ to the student empirical risk minimization problem in \eqref{erm} with a fixed teacher $\p^t$ satisfies:
\switchreptheorem{
  \begin{equation*}
    R(\hat{h}) - R(h^*) \le 
    \underbrace{\mathcal{O}\left(\sqrt{\frac{\log(|\cH|/\delta)}{n^u}}\right)}_{\text{Student estimation error}}
    + \underbrace{\min_{h\in \cH}R^t(h) - \min_{h: \cX \>[m]}R^t(h)}_{\text{Student approximation error}}
    + \underbrace{\expectation_{x}\left[\|\p^t(x) - \p^\phi(x)\|_1\right]}_{\text{Teacher approximation error}},
  \end{equation*}
}{
  \begin{equation*}
    R(\hat{h}) - R(h^*) \le 
    \underbrace{\mathcal{O}\left(\sqrt{\frac{\log(|\cH|/\delta)}{n^u}}\right)}_{\text{Student estimation error}}
    + \underbrace{\min_{h\in \cH}R^t(h) - \min_{h: \cX \>[m]}R^t(h)}_{\text{Student approximation error}}
    + \underbrace{\expectation_{x}\left[\|\p^t(x) - \p^\phi(x)\|_1\right]}_{\text{Teacher approximation error}},
  \end{equation*}
}
where $|\cH|$ can be replaced by a measure of capacity of the student hypothesis class $\cH$.
\end{thm}

One can then plug into \thmref{excess-error-bound} a suitable measure of capacity for student hypothesis class $\cH$. For example, the following corollary uses the Natarajan dimension  of $\cH$~\citep{Natarajan:1989} to bound its capacity.
\begin{cor}{excess-error-bound-cor}
Under the assumptions in \thmref{excess-error-bound}, for a student hypothesis class $\cH$ with Natarajan dimension $d$~\citep{Natarajan:1989}, with  probability at least $1 - \delta$ over draw of $n^u$ unlabeled examples from $D_\cX$, the solution $\hat{h} \in \cH$ to the student empirical risk minimization problem in \eqref{erm} with a fixed teacher $\p^t$ satisfies.
  \begin{equation*}
    R(\hat{h}) - R(h^*) \le 
    {\mathcal{O}\left(\sqrt{\frac{d\log(mn^u) + \log(1/\delta)}{n^u}}\right)}
    + {\min_{h\in \cH}R^t(h) - \min_{h: \cX \>[m]}R^t(h)}
    + {\expectation_{x}\left[\|\p^t(x) - \p^\phi(x)\|_1\right]}.
  \end{equation*}
\end{cor}

For proving \thmref{excess-error-bound} (and \corref{excess-error-bound-cor}), we will need the following confidence bound on the teacher-distilled risk $R^t$ for a student $h \in \cH$ in terms of its empirical risk. 
\begin{lem}{generalization-bound}
Let $\p^t$ be a fixed teacher. Let $\hat{R}^t(h)$ denote the empirical teacher-distilled for a student $h$ on $n^u$ unlabeled examples $S^u$, labeled by the teacher $p^t$:
\[
\hat{R}^t(h) = \frac{1}{n^u}\sum_{x \in S^u}\sum_{i=1}^m p^t_i(x)\1(h(x) \ne i).
\]
Let the expected teacher-distilled risk $R^t(h)$ for the same student be as defined in \eqref{distilled-risk}. 
Fix $\delta \in (0,1)$. Then with  probability at least $1 - \delta$ over draw of the $n^u$ unlabeled examples from $D_\cX$, for any $h \in \cH$:
\[
|{R}^t(h) - \hat{R}^t(h)| \leq \mathcal{O}\left(\sqrt{\frac{\log(|\cH|/\delta)}{n^u}}\right),
\]
where $|\cH|$ can be replaced by a measure of capacity of the student hypothesis class $\cH$. For example, if $\cH$ has a Natarajan dimension $d$~\citep{Natarajan:1989}, then with with  probability at least $1 - \delta$  over draw of the $n^u$ unlabeled examples from $D_\cX$, for any $h \in \cH$:
\[
|{R}^t(h) - \hat{R}^t(h)| \leq \mathcal{O}\left(\sqrt{\frac{d\log(mn^u) + \log(1/\delta)}{n^u}}\right).
\]
\end{lem}
\begin{prf}{generalization-bound}
For a fixed $h$, we can straightforwardly apply  Hoeffding's inequality to show that with probability at least $1-\delta$, $|{R}^t(h) - \hat{R}^t(h)| \leq \mathcal{O}\big(\sqrt{{\log(1/\delta)}/{n^u}}\big)$, where we have used the fact that each $p^t_i(x) \in [0,1]$  and $\sum_i p^t_i(x) = 1$. If $\cH$ is finite, one can further take a union bound over all $h \in \cH$ and show that  with probability at least $1-\delta$,  $|{R}^t(h) - \hat{R}^t(h)| \leq \mathcal{O}\big(\sqrt{{\log(|\cH|/\delta)}/{n^u}}\big)$. For an infinite class $\cH$, one typically  replaces the size of $\cH$ with its ``growth function'' for a given number of examples $n^u$~\citep{Daniely:2015}, and further upper bounds the growth function in terms of a capacity term. For instance, we have from Theorem 13 in \citet{Daniely:2015} that when $\cH$ has Natarajan dimension $d$~\citep{Natarajan:1989},  $|{R}^t(h) - \hat{R}^t(h)| \leq \mathcal{O}\Big(\sqrt{\big({d\log(mn^u) + \log(1/\delta)}\big)/{n^u}}\Big)$, as desired.
\end{prf}

Equipped with the generalization bound in \lemref{generalization-bound}, we are now ready to prove \thmref{excess-error-bound}.
\allowdisplaybreaks
\begin{prf}{excess-error-bound}
%
We will use $\bar{h} \in \cH$ to denote a student model 
which minimizes the teacher-distilled risk $R^t$, \ie 
for which $R^t(\bar{h}) \leq R^t(h), \forall h \in \cH$.
Expanding the excess risk for $\hat{h}$ using \eqref{pop-risk}, we have:
%
\begin{align*}
  \MoveEqLeft R(\hat{h}) \,-\, R(h^*) \\
  &=
  \expectation_x\left[\sum_{i=1}^m p^*_i(x)\left(\1(\hat{h}(x) \ne i) - \1(h^*(x) \ne i)\right)\right] \\
  &=
  \expectation_x\left[p^*_{h^*(x)}(x) - p^*_{\hat{h}(x)}(x)\right] \\
  &=\expectation_x\left[\max_i p^*_{i}(x) - p^*_{\hat{h}(x)}(x)\right]
  ~~~~\text{(from the definition of $h^*(x)$)} \\
  &\leq \expectation_x\left[\max_i p^\phi_{i}(x) - p^\phi_{\hat{h}(x)}(x)\right]
  ~~~~\text{(from \eqref{margin})} \\
  &=
  \expectation_x\left[p^\phi_{h^*(x)}(x) - p^\phi_{\hat{h}(x)}(x)\right]
  ~~~~\text{(from \eqref{argmax}, $h^*(x)$ is also the maximizing index for $\p^\phi(x)$)} \\
  &=
  \expectation_x\left[\sum_{i=1}^m p^\phi_i(x)\left(\1(\hat{h}(x) \ne i) - \1(h^*(x) \ne i)\right)\right] \\
  &=
  \expectation_x\left[\sum_{i=1}^m p^t_i(x) \left(\1(\hat{h}(x) \ne i) - \1(h^*(x) \ne i)\right)\right]
  \,+\,
  \expectation_x\left[\sum_{i=1}^m (p^\phi_i(x) - p^t_i(x))  \left(\1(\hat{h}(x) \ne i) - \1(h^*(x) \ne i)\right)\right]
  \\
  &\leq
  \expectation_x\left[\sum_{i=1}^m p^t_i(x)\left(\1(\hat{h}(x) \ne i) - \1(h^*(x) \ne i)\right)\right]\\
  &
  \hspace{1.67cm}
  \,+\, \expectation_x\left[\|\p^t(x) - \p^\phi(x)\|_1\max_i\left|\1(\hat{h}(x) \ne i) - \1(h^*(x) \ne i)\right|\right]\\ 
  &\hspace{9cm}\text{(using H\"{o}lder's inequality)} \\
    &\leq
  \expectation_x\left[\sum_{i=1}^m p^t_i(x)\left(\1(\hat{h}(x) \ne i) - \1(h^*(x) \ne i)\right)\right]
  \,+\, \expectation_x\left[\|\p^t(x) - \p^\phi(x)\|_1\right]\\
  &=
  R^t(\hat{h}) \,-\, R^t(h^*)
  \,+\, \expectation_x\left[\|\p^t(x) - \p^\phi(x)\|_1\right]~~~~\text{(from \eqref{distilled-risk})} \\
  &\leq
  R^t(\hat{h}) \,-\, \min_{h:\cX\>[m]}R^t(h)
  \,+\, \expectation_x\left[\|\p^t(x) - \p^\phi(x)\|_1\right] \\
  &=
  R^t(\hat{h}) \,-\, \min_{h\in \cH}R^t(h) \,+\,
  \min_{h\in \cH}R^t(h) \,-\,
  \min_{h:\cX\>[m]}R^t(h^*)
  \,+\, \expectation_x\left[\|\p^t(x) - \p^\phi(x)\|_1\right] \\
  &=
  R^t(\hat{h}) \,-\, R^t(\bar{h}) \,+\,
  \min_{h\in \cH}R^t(h) \,-\,
  \min_{h:\cX\>[m]}R^t(h^*)
  \,+\, \expectation_x\left[\|\p^t(x) - \p^\phi(x)\|_1\right] \\
  &\hspace{9cm}\text{(from definition of $\bar{h}$)}
  \\
  &=
  R^t(\hat{h}) \,-\, 
  \hat{R}^t(\hat{h}) \,+\, \hat{R}^t(\hat{h})
  \,-\,
  R^t(\bar{h}) \,+\,
  \min_{h\in \cH}R^t(h) \,-\,
  \min_{h:\cX\>[m]}R^t(h^*)
  \,+\, \expectation_x\left[\|\p^t(x) - \p^\phi(x)\|_1\right] \\
  &\leq
  R^t(\hat{h}) \,-\, 
  \hat{R}^t(\hat{h}) \,+\, \hat{R}^t(\bar{h})
  \,-\,
  R^t(\bar{h}) \,+\,
  \min_{h\in \cH}R^t(h) \,-\,
  \min_{h:\cX\>[m]}R^t(h^*)
  \,+\, \expectation_x\left[\|\p^t(x) - \p^\phi(x)\|_1\right]\\
  &\hspace{9cm}\text{(because $\hat{h}$ minimizes the empirical risk $\hat{R}^t$)}
  \\
  &\leq
  2\sup_{h\in\cH}|R^t(h) \,-\, 
  \hat{R}^t(h)| \,+\,
  \min_{h\in \cH}R^t(h) \,-\,
  \min_{h:\cX\>[m]}R^t(h^*)
  \,+\, \expectation_x\left[\|\p^t(x) - \p^\phi(x)\|_1\right] \\
  &\leq
  \mathcal{O}\left(\sqrt{\frac{\log(|\cH|/\delta)}{n^u}}\right)
  \,+\, 
  \min_{h\in \cH}R^t(h) \,-\, \min_{h: \cX \>[m]}R^t(h) \,+\,
  \expectation_{x}\left[\|\p^t(x) - \p^\phi(x)\|_1\right],
\end{align*}
where the last statement holds with probability $\geq 1-\delta$ over draw of $S^u$, and follows from the first bound in \lemref{generalization-bound}.
%
\end{prf}

To prove \corref{excess-error-bound-cor}, we instead apply the second bound from \lemref{generalization-bound} in the last step.

\subsection{Examples of Margin-preserving Transformations}\label{app:temperature-scaling}
\begin{lem}{temp-scale}
The following transformation functions satisfy the margin conditions stated  in \eqrefs{argmax}{margin} in \thmref{excess-error-bound}:
\begin{enumerate}
    \item $\phi_i(z) = \1(\argmax_j z_j = i)$
    \item  $\phi_i(z)  = \frac{z_i^\alpha}{\sum_j z_j^\alpha}$ with $\alpha > 1$.
\end{enumerate}
\end{lem}
\begin{proof}
For the first transformation, $\phi_i(z) = \1(\argmax_j z_j = i)$, \eqref{argmax} is trivially true, and because $\max_i\phi_i(z) - \phi_j(z) \in \{0,1\}, \,\forall j \in [m]$, \eqref{margin} also holds. For the second transformation, $\phi_i(z)  = z_i^\alpha / \sum_j z_j^\alpha$ with $\alpha > 1$, \eqref{argmax} follows from  $z_i > z_j \iff z_i^\alpha > z_j^\alpha$ and $z_i = z_j \iff z_i^\alpha = z_j^\alpha$. We now show that \eqref{margin} also holds for this transformation. Fix $z \in \Delta_m$, and let $i^* = \argmax_i z_i$, with ties broken in favor of the larger index. Note that $z^\alpha_{i^*} \geq z^\alpha_j, \forall j \in [m].$ We then have for any $j \in [m]$,
\begin{eqnarray*}
\max_i\phi_{i}(z) - \phi_{j}(z) &=&
\frac{z_{i^*}^\alpha - z_{j}^\alpha}{\sum_k z_k^\alpha}~\geq~\frac{z_{i^*}^\alpha - z_{j}^\alpha}{\sum_k z_{i^*}^{\alpha-1}z_k}~=~\frac{z_{i^*}^\alpha - z_{j}^\alpha}{ z_{i^*}^{\alpha-1}\sum_kz_k}~=~\frac{z_{i^*}^\alpha - z_{j}^\alpha}{ z_{i^*}^{\alpha-1}(1)}\\
&=&z_{i^*} - z_{j}\left(\frac{z_j}{z_{i^*}}\right)^{\alpha-1}
~\geq~ z_{i^*} - z_{j}\left(1\right)^{\alpha-1}~=~ z_{i^*} - z_{j}
~=~ \max_i z_{i} - z_{j},~~\text{as desired.}
\end{eqnarray*}
\end{proof}
\section{Additional Experimental Details}\label{app:experiments}

\begin{figure*}[t]

\centering

\begin{tabular}{cc}
    \includegraphics[width=0.45\textwidth]{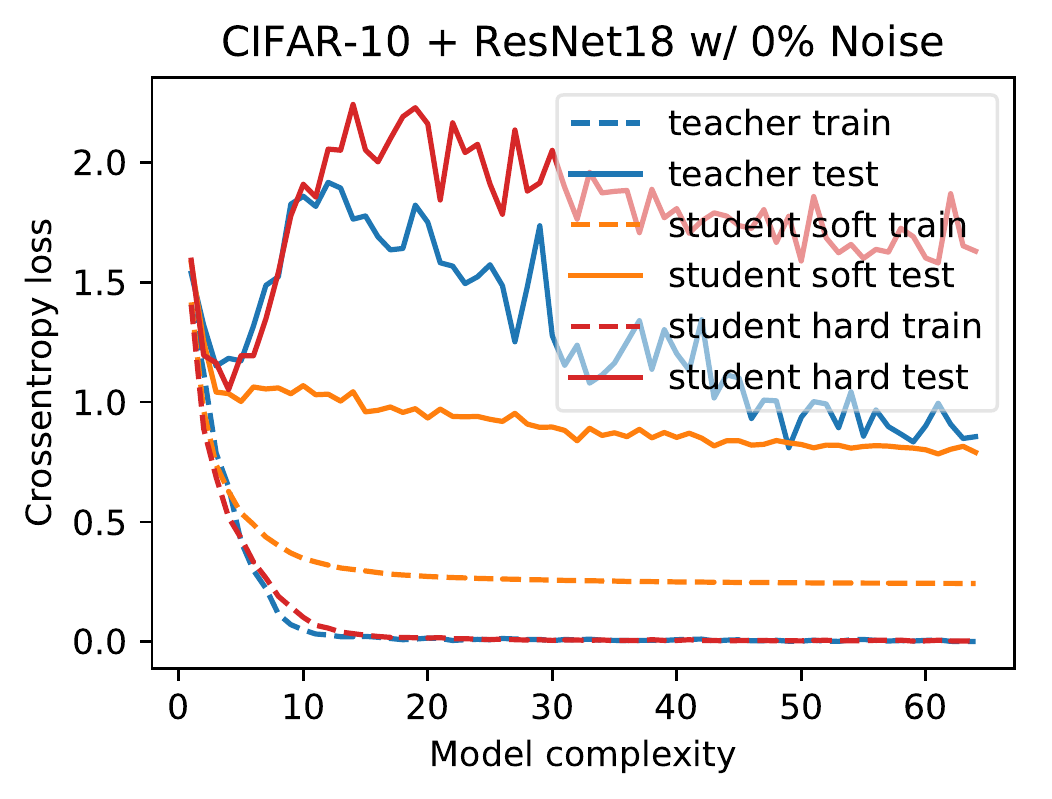} 
    & \includegraphics[width=0.45\textwidth]{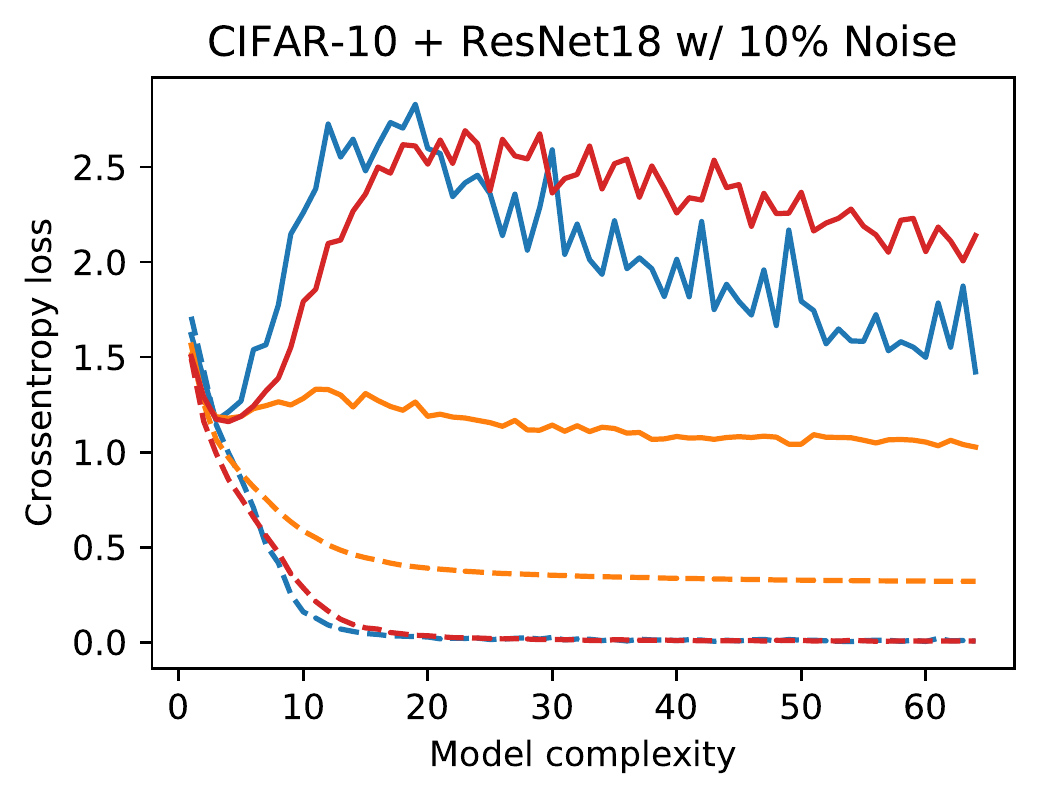} \\
    \includegraphics[width=0.45\textwidth]{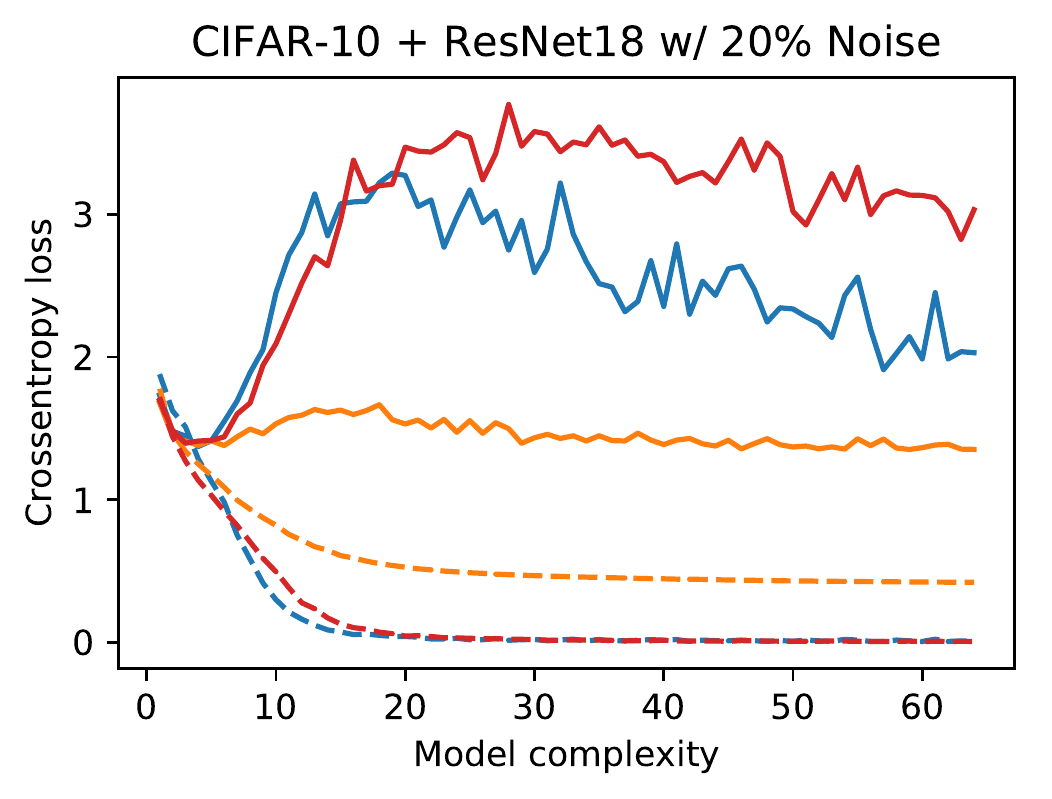} 
    & \includegraphics[width=0.45\textwidth]{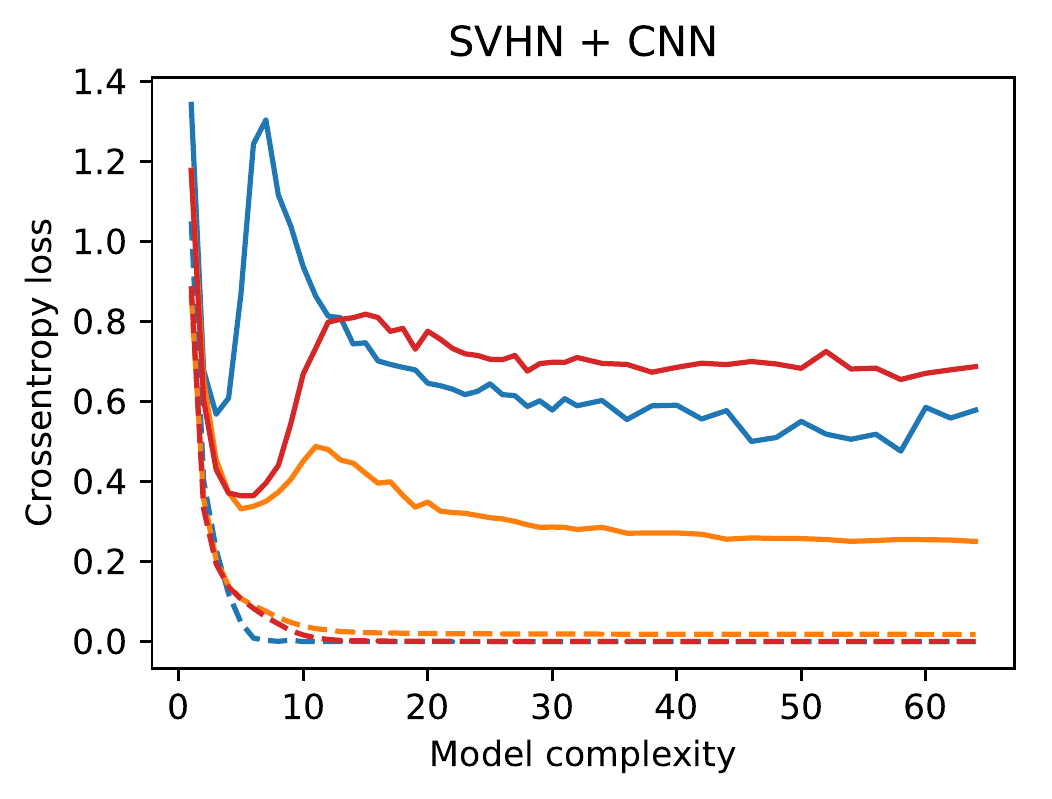} \\
    \includegraphics[width=0.45\textwidth]{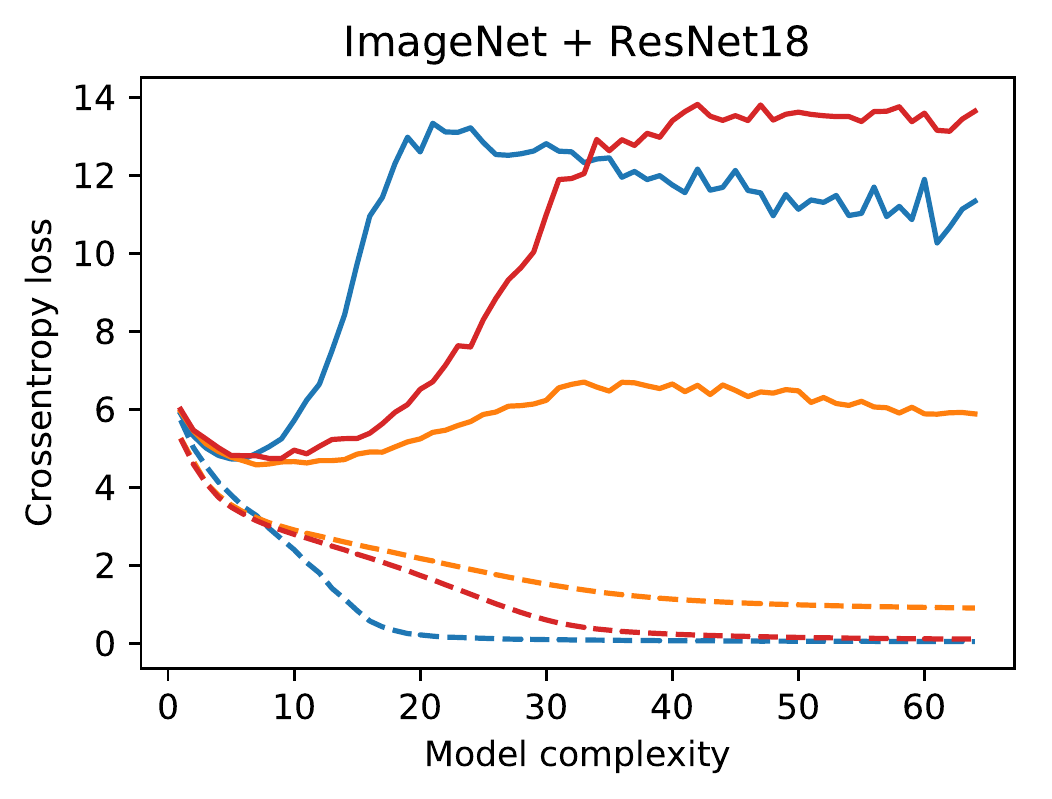} 
    & \includegraphics[width=0.45\textwidth]{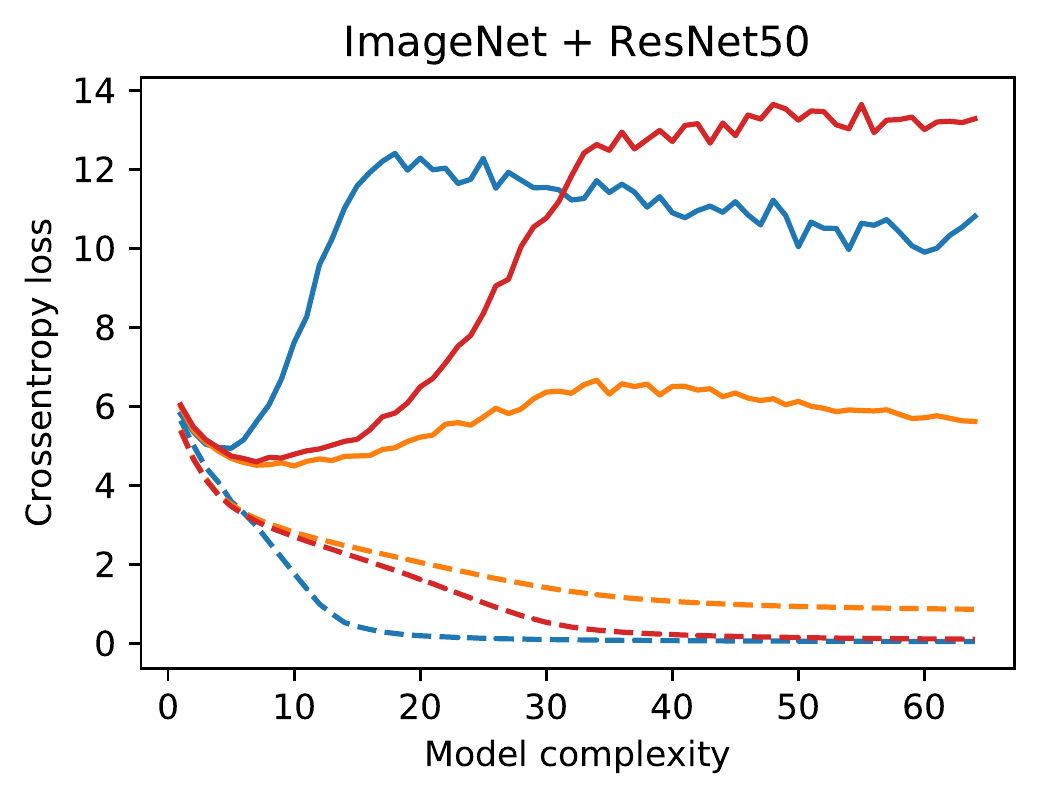} \\
    
\end{tabular}

\caption{
  Plots of the training and testing cross-entropy losses for all of our experiments (the extra two CIFAR-10 experiments are described in \appref{experiments:noise}). These are included mostly for completeness, but it's notable that, beyond a certain point, the testing cross-entropy losses of the teacher \emph{decrease} as a function of model complexity, which is further evidence that highly overparameterized models can generalize better.
}

\label{fig:train-test-loss}

\end{figure*}
\textbf{CIFAR-10}: We used the standard CIFAR-10 train and test splits, while the train split was further randomly split into three folds - one fold was used to train the teacher and the other two folds were labeled by the teacher to train the students.

During teacher training, we injected label noise to the training data by randomly flipping 0\% (i.e. no noise), 10\% and 20\% of the labels to a wrong class with equal chances. We also augmented the 32-pixel by 32-pixel input images by (1) adding a 4-pixel padding along each side and randomly cropping back to 32 pixels, and (2) flipping vertically with 1/2 probability. 

\textbf{SVHN}: We recombined the train and the extra splits of SVHN (\texttt{svhn\_cropped}) from \citet{TFDS}, subsampled 12\% examples to train the teacher and asked the teacher to label the rest 88\% examples to train the students. We also transformed the input images to grayscale before passing them to the CNN. 

The 5-layer CNN models consist of 4 convolutional layers and a final dense layer. Each convolutional layer applies 2D convolution, batch normalization, ReLU activation and max pooling sequentially.   

\textbf{ImageNet}: We worked with the resized ImageNet (\texttt{imagenet\_resized/64x64}) dataset \citep{chrabaszcz2017downsampled} containing 64-pixel by 64-pixel images. We applied 20\% of the standard training examples to teacher training and used the rest 80\% examples for labeling and student training. 

\subsection{Memorization of Training Sets}

\begin{figure*}[t]

\centering

\begin{tabular}{cc}
    \includegraphics[width=0.45\textwidth]{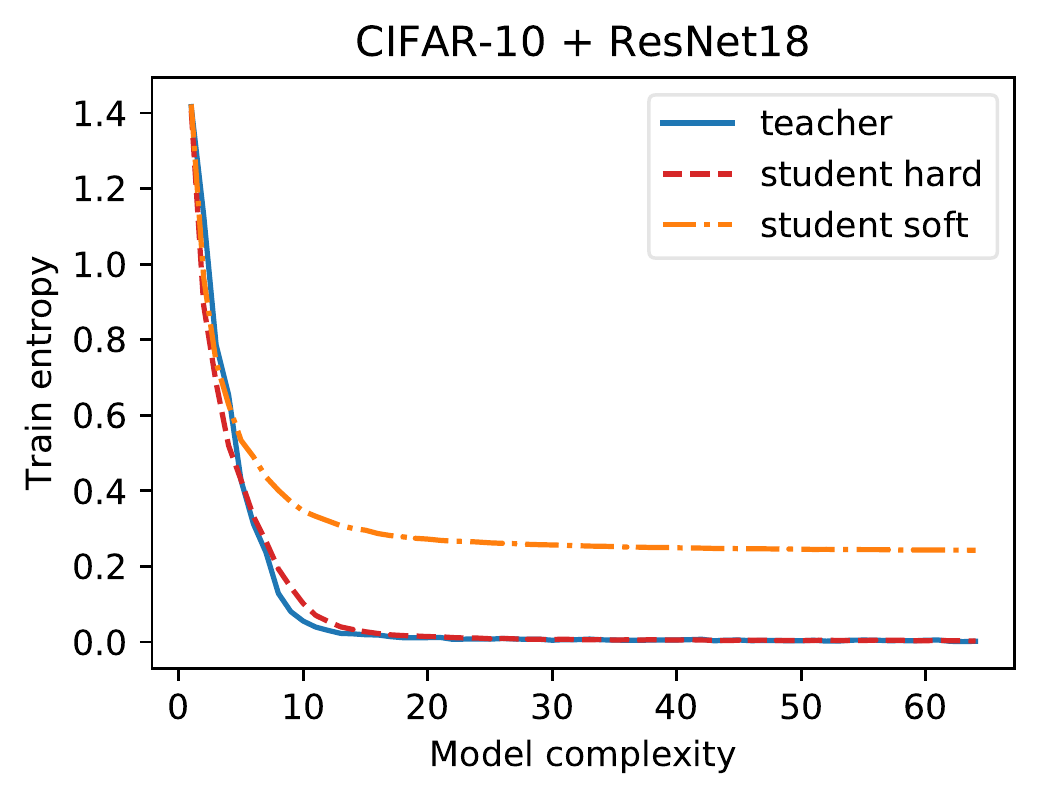} 
    & \includegraphics[width=0.45\textwidth]{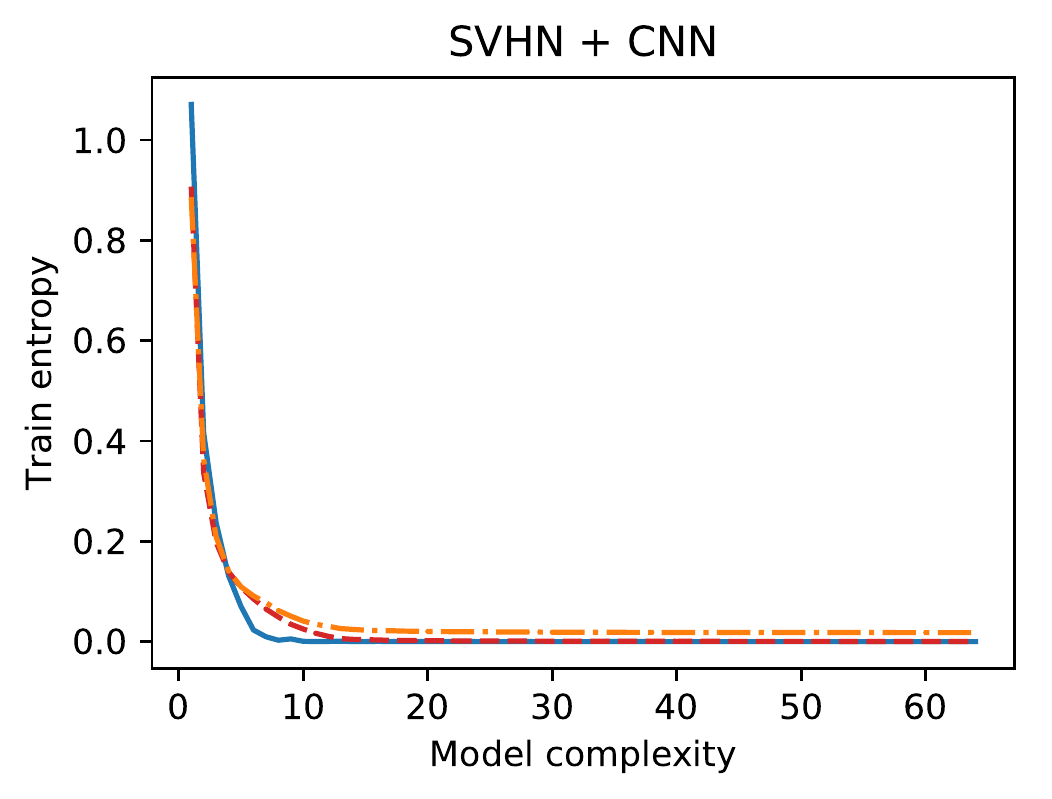} \\
    \includegraphics[width=0.45\textwidth]{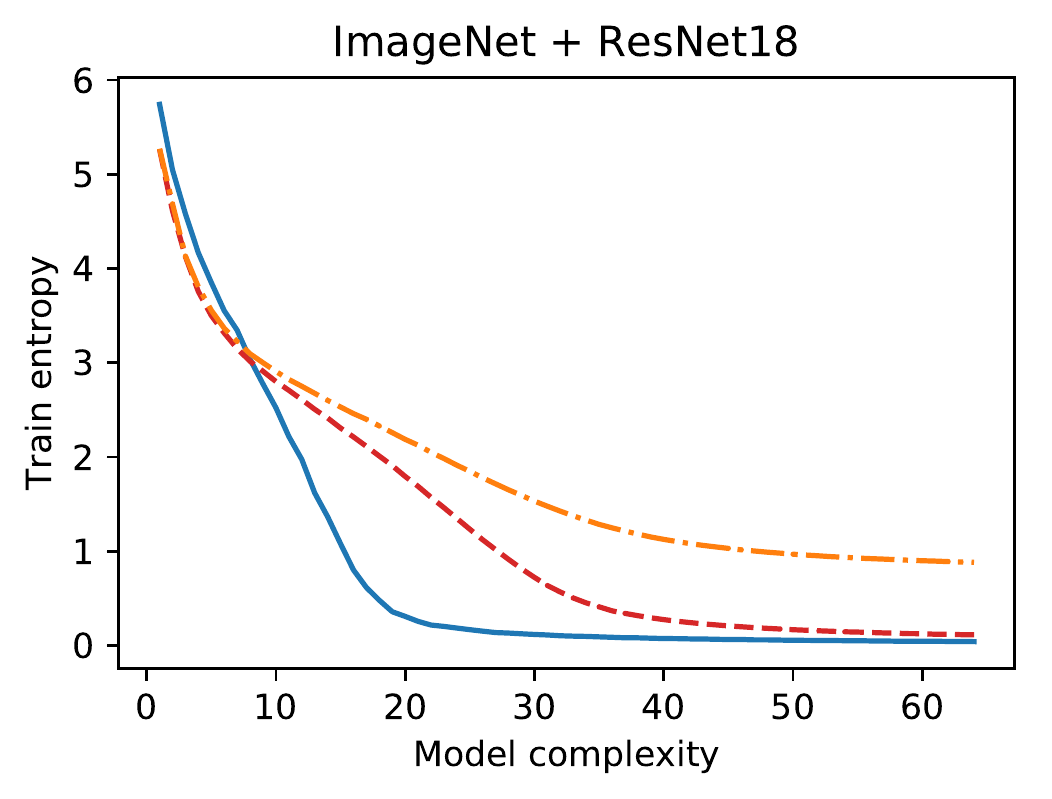} 
    & \includegraphics[width=0.45\textwidth]{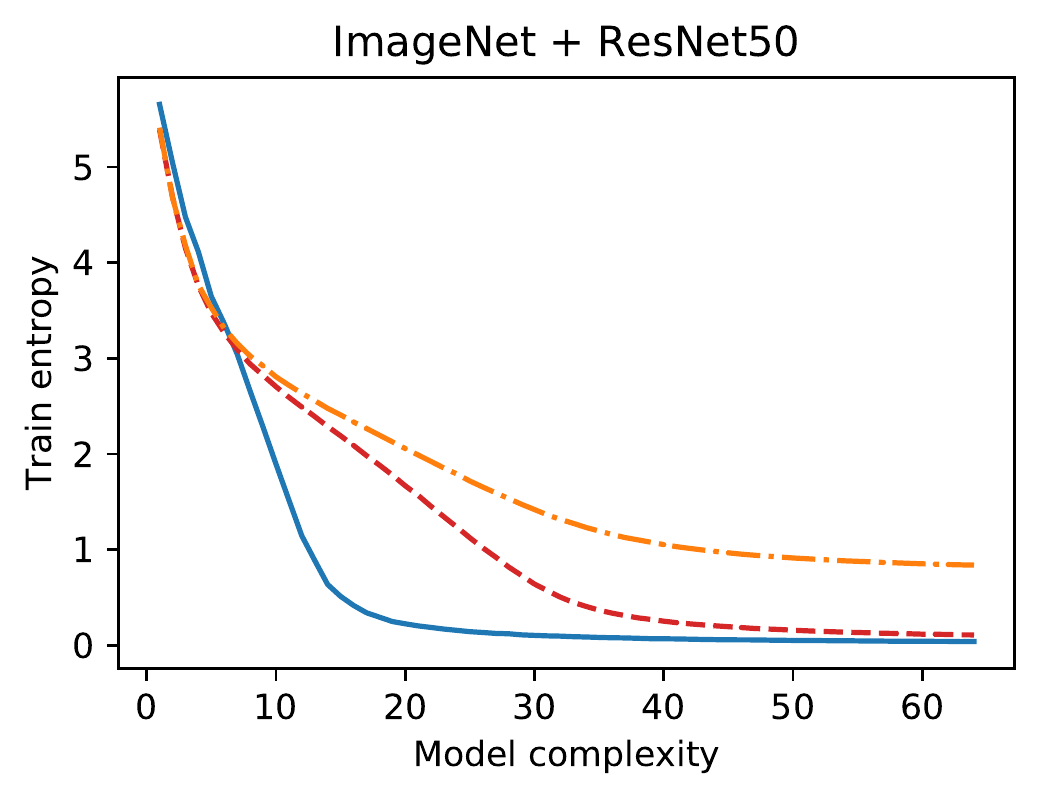}
\end{tabular}

\caption{
  Average entropy of the predictions of the teacher and two students on their respective training sets. Lower values indicate a ``harder'' classifier, and higher values a ``softer'' one. The two models trained on hard labels eventually memorize their training data.
  Clockwise from top-left: CIFAR-10 using ResNet18 models, SVHN using 5-layer CNN models, ImageNet using ResNet50 models, and ImageNet using ResNet18 models. Teacher-train/student-train/test splits can be found in \tabref{experiment-setup}.
}

\label{fig:train-entropy}

\end{figure*}
\figref{train-entropy} shows plots of the average entropy of the (soft) predictions made by the teacher and the two students on their respective training sets (for the students, these are the held-out datasets labeled by the teacher), in each of our four experiments. Both the teacher, and the student trained on hard labels, eventually memorize their training data, although it generally takes the student somewhat longer, presumably because it has a larger dataset. The student trained on soft labels cannot, of course, truly ``memorize'' them, since the labels themselves are soft.

\subsection{Artificial Label Noise on CIFAR-10}\label{app:experiments:noise}

\begin{figure*}[t]

\centering

\begin{tabular}{ccc}
    \includegraphics[width=0.3\textwidth]{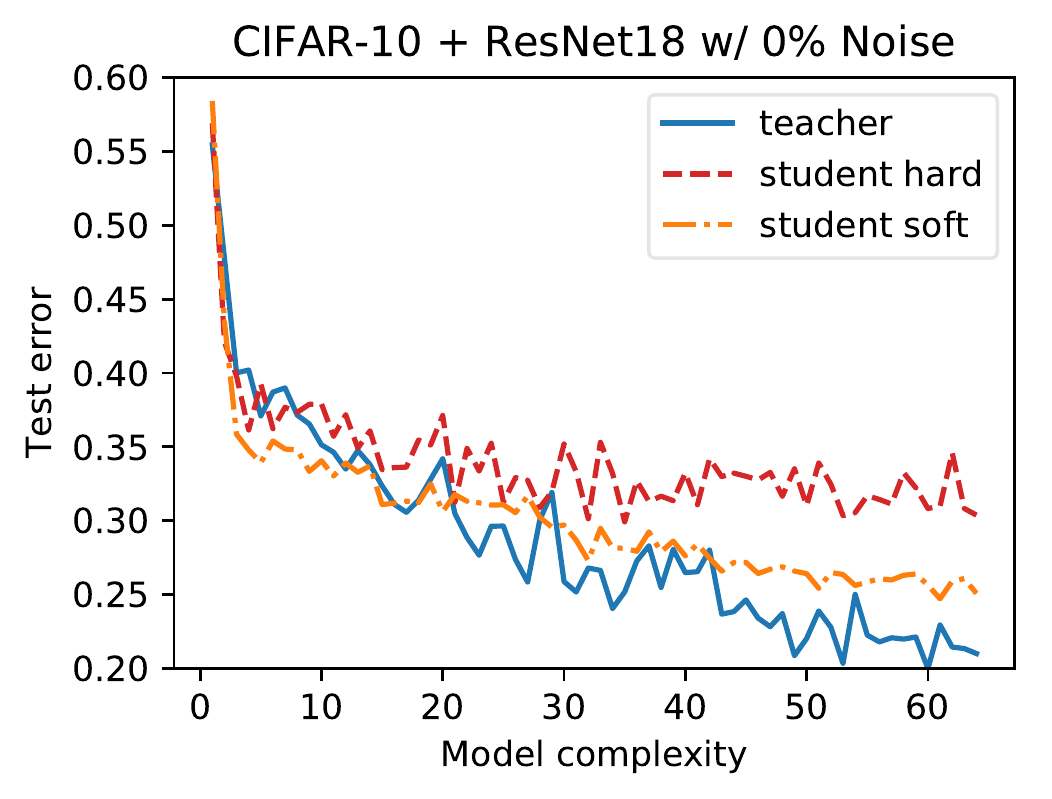} 
    & \includegraphics[width=0.3\textwidth]{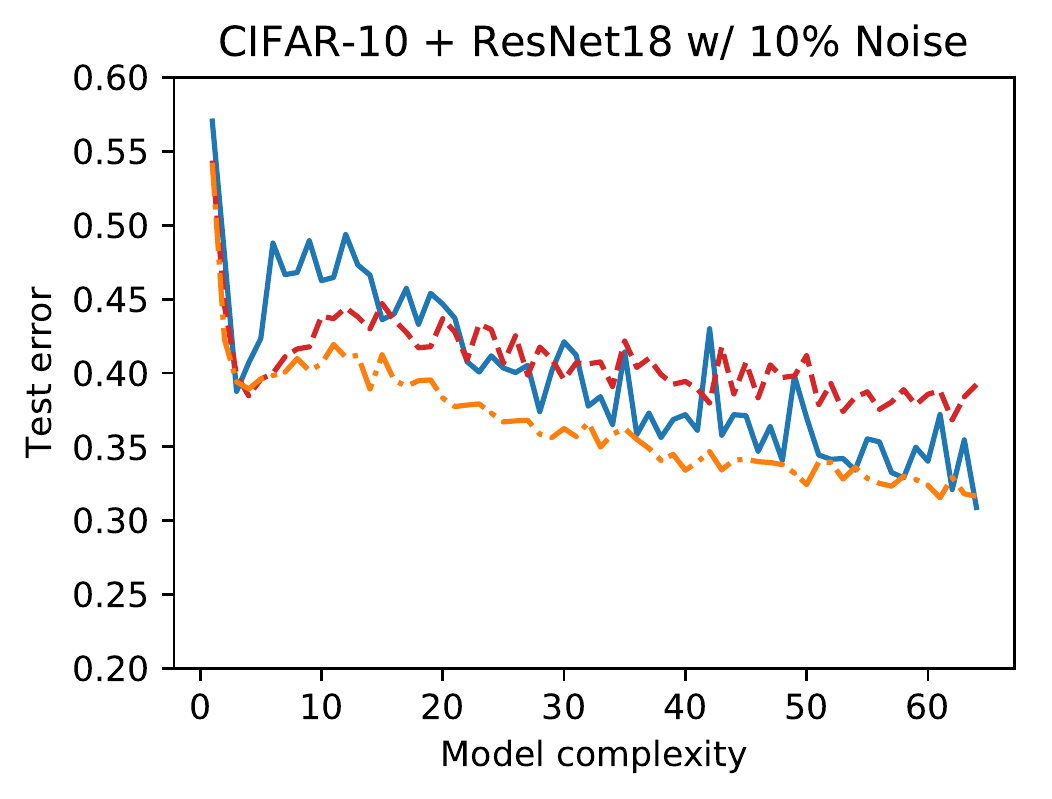} &
    \includegraphics[width=0.3\textwidth]{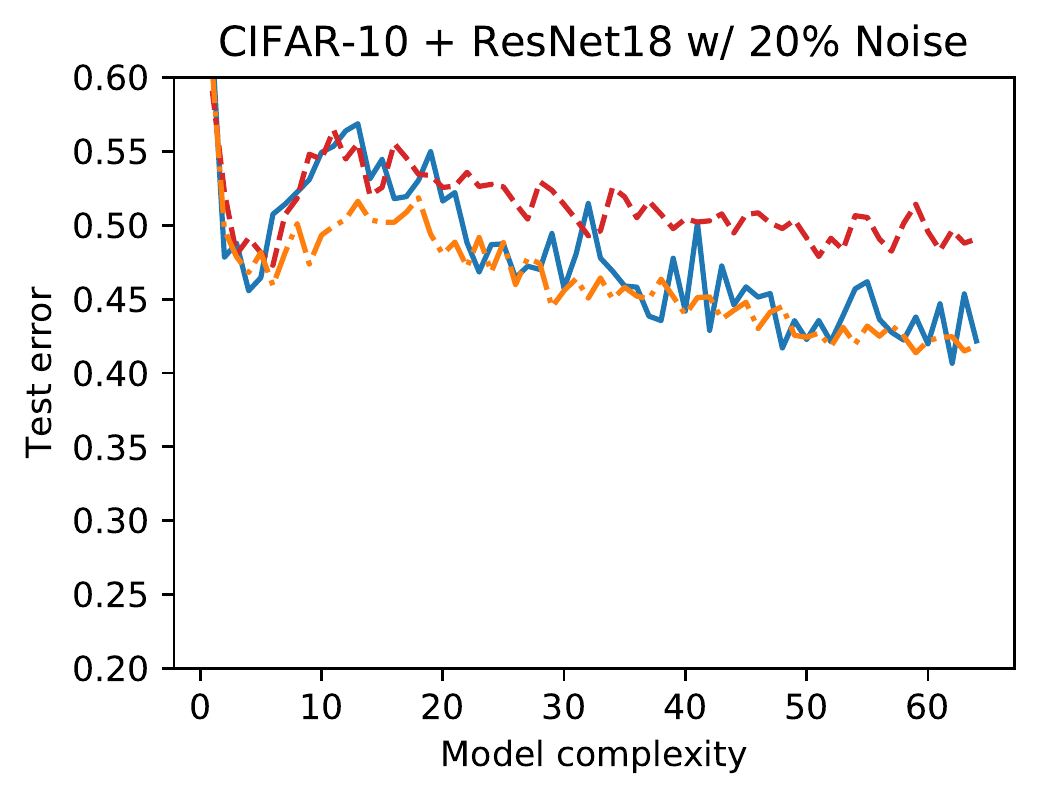} 
\end{tabular}

\caption{
  CIFAR-10 test error rates as a function of model complexity $m$ for a ``teacher'' model trained on the teacher training set, and for two student models trained on a dataset labeled by the teacher. One student was trained from ``soft'' teacher labels, and the other on ``hard'' labels derived from thresholding the teacher.
  From left to right: no noise, $10\%$ \iid label noise on the teacher's training set, and $20\%$ noise. The teacher-train/student-train/test split can be found in \tabref{experiment-setup}.
}

\label{fig:cifar10-noise}

\end{figure*}
On the CIFAR-10 experiment of \secref{experiments}, it appeared that the ``classical'' regime was too brief to enable the student trained on hard labels to outperform the teacher (although the student trained on soft labels did, albeit only for extremely low $m$s).

\figref{cifar10-noise} explores what happens on this dataset if we introduce artificial \iid label noise to the training set provided to the teacher (and \emph{only} this dataset). Naturally, adding noise increases the error rates, but it \emph{also} makes the problem more difficult, thereby pushing the ``hump'' in the double-descent curve to the right. We can see that, once this happens, our desired phenomenon becomes more pronounced: while it is still debatable whether the student trained on hard labels outperforms the teacher, the student trained on soft labels does so very clearly.

\subsection{Reproducibility}\label{app:experiments:reproducibility}

Since we applied no early stopping, and trained all models for an excessive amount of time in order to ensure that we entered the ``memorization'' regime, we expect to observe little (or none) of the systematic randomness caused by under-fitting. Having observed that averaging over multiple experimental runs required intensive resources, but did not uncover additional results, we chose to report single experiment runs in this paper. 

\figref{cifar10-stability} provides empirical evidence for this observation, on CIFAR-10. We repeated the student training from the $m=64$ teacher five times, and plotted the average errors and their ranges. We can see that the test error pattern is well-captured by the randomly-chosen single run.

\begin{figure*}[t]

\centering

\begin{tabular}{cc}
    \includegraphics[width=0.45\textwidth]{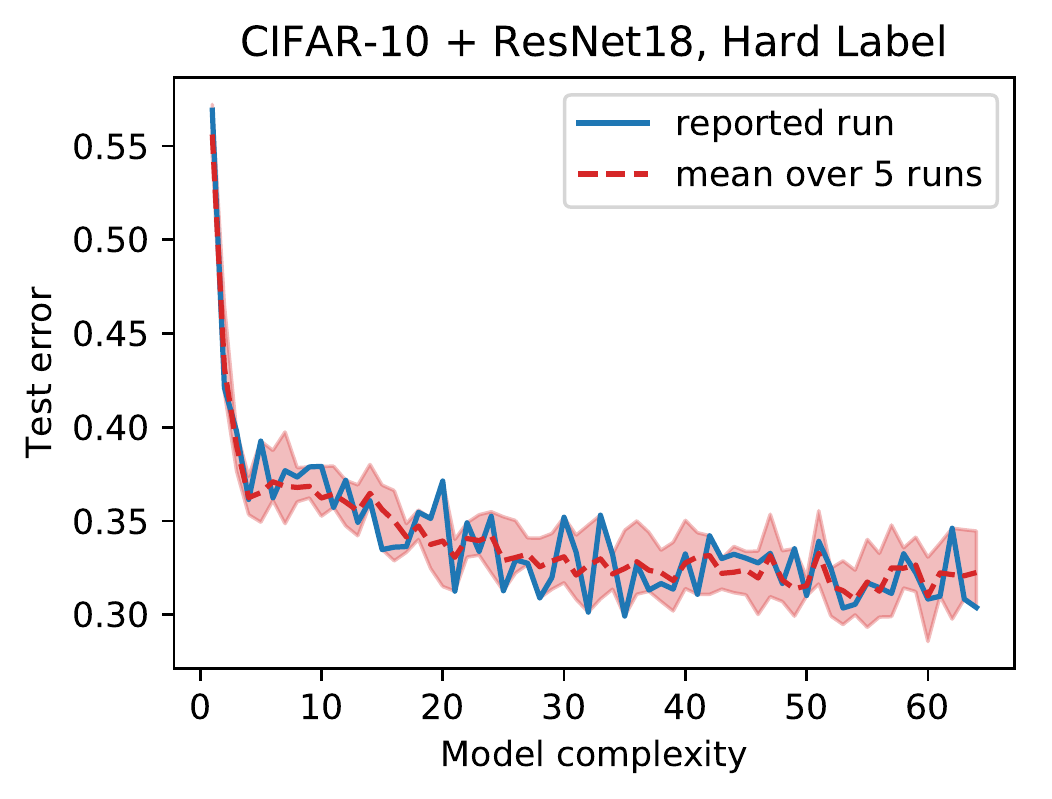} 
    & \includegraphics[width=0.45\textwidth]{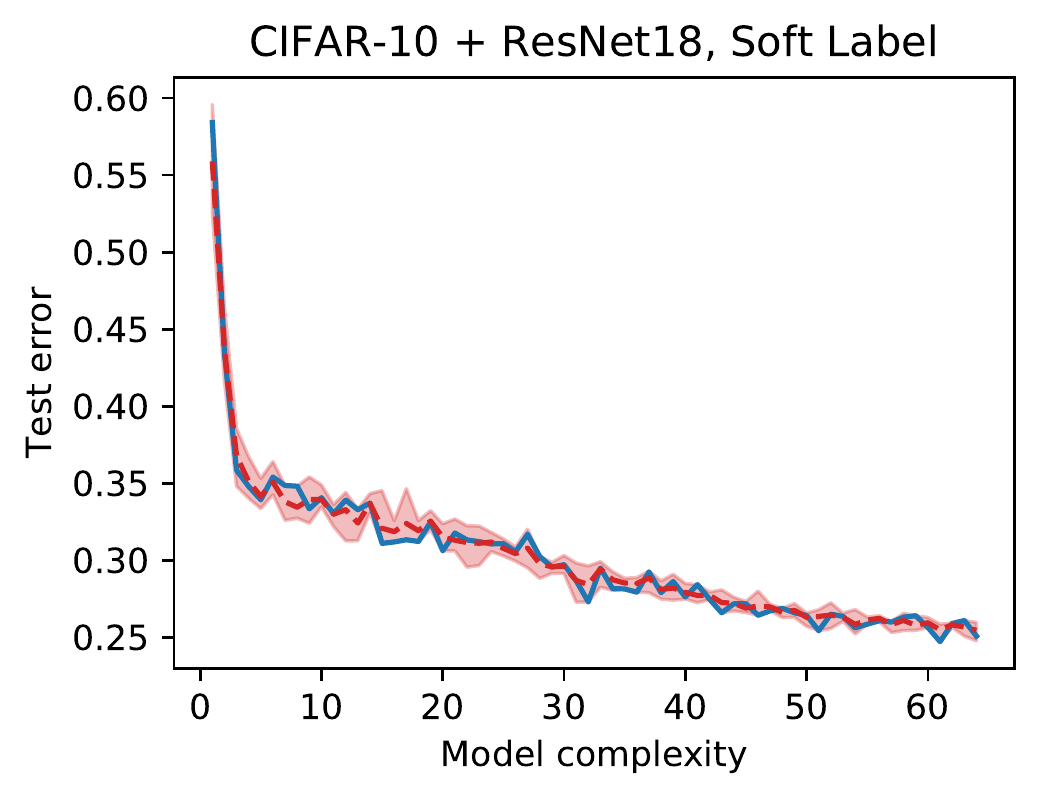} 
\end{tabular}
\caption{
  CIFAR-10 test error rates as a function of model complexity $m$. The blue curves correspond to the single runs we reported in the main body, e.g. in \figref{experiments}, while the red curves and bands represent the mean and the range of 5 repeated independent runs under the same settings. In both cases the randomly-chosen reported run captures the overall behavior of all involved curves. 
}

\label{fig:cifar10-stability}

\end{figure*}

\label{document:end}

\end{document}